\documentclass{article}

\usepackage[utf8]{inputenc}

\usepackage{amsmath}
\usepackage[noblocks]{authblk}
\usepackage{natbib}
\usepackage{hyperref} 
\usepackage{amssymb}
\usepackage{verbatim}
\usepackage{amsthm}

\usepackage[left=2.5cm,top=3cm,right=2.5cm,bottom=3cm,bindingoffset=0.5cm]{geometry}

\usepackage{graphicx} 
\usepackage{caption}
\usepackage{subcaption}
\usepackage{booktabs}

\usepackage{algorithm}
\usepackage[noend]{algorithmic} 

\newtheorem{definition}{Definition}[section]
\newtheorem{hypothesis}{Hypothesis}[section]
\newtheorem{theorem}{Theorem}

\DeclareMathOperator{\sign}{sign}
\DeclareMathOperator{\Tr}{Tr}

\DeclareMathOperator{\SVD}{SVD}
\DeclareMathOperator{\diag}{diag}

\begin{document}

\title{On the Needs for Rotations in Hypercubic Quantization Hashing}

\author[1,2]{Anne Morvan \footnote{To whom correspondence should be adressed: anne.morvan@cea.fr. Partly supported by the \textit{Direction G\'en\'erale de l'Armement} (French Ministry of Defense).}}
\affil[1]{CEA, LIST, 91191 Gif-sur-Yvette, France}
\affil[2]{Universit\'e Paris-Dauphine, PSL Research University, CNRS, LAMSADE, 75016 Paris, France}
\author[1]{Antoine Souloumiac}
\author[3]{Krzysztof Choromanski}
\affil[3]{Google Brain Robotics, New York, USA}
\author[1]{C\'edric Gouy-Pailler}
\author[2]{Jamal Atif}

\date{February 12, 2018}

\maketitle

\begin{abstract}
The aim of this paper is to endow the well-known family of hypercubic quantization hashing methods with theoretical guarantees.
In hypercubic quantization, applying a suitable (random or learned) rotation after dimensionality reduction has been experimentally shown to improve the results accuracy in the nearest neighbors search problem. 
We prove in this paper that the use of these rotations is optimal under some mild assumptions: getting optimal binary sketches is equivalent to applying a rotation uniformizing the diagonal of the covariance matrix between data points. Moreover, for two closed points, the probability to have dissimilar binary sketches is upper bounded by a factor of the initial distance between the data points. Relaxing these assumptions, we obtain a general concentration result for random matrices.
We also provide some experiments illustrating these theoretical points and compare a set of algorithms in both the batch and online settings.
\end{abstract}

\section{Introduction}
\label{sec:intro}
Nearest neighbors (NN) search is a key task involved in many machine learning applications such as classification or clustering. For large-scale datasets {\em e.g.} in computer vision or metagenomics, indexing efficiently high-dimensional data becomes necessary for reducing space needs and speed up similarity search. This can be classically achieved by hashing techniques which map data onto lower-dimensional representations. 
Two hashing paradigms exist: data-independent and data-dependent hashing methods.
On the one hand, Locality-Sensitive Hashing (LSH)~(\cite{Andoni2008NHA}) and its variants~(\cite{terasawa2007, indyk2015, CBEICML14}) belong to the data-independent paradigm. They rely on some random projection onto a $c$-lower dimensional space followed by a scalar quantization returning the nearest vertex from the set $\{-1, 1\}^c$ for getting the binary codes (e.g. the $\sign$ function is applied point-wise).
On the other hand, data-dependent methods~(\cite{survey18}) learn this projection from data instead and have been found to be more accurate for computing similarity-preserving binary codes.
Among them, the unsupervised data-dependent hypercubic hashing methods, embodied by ITerative Quantization (ITQ)~(\cite{ITQ2013}), use Principal Component Analysis (PCA) to reduce data dimensionality to $c$: the data is projected onto the first $c$ principal components chosen as the ones with highest explained variance as they carry more information on variability. If we then directly mapped each resulting direction to one bit, each of them would get represented by the same volume of binary code ($1$ bit), although the $c^{th}$ direction should carry less information than the first one. Thus, one can intuitively understand why PCA projection application solely leads to poor performance of obtained binary codes in the NN search task. This is why data get often mixed though an isometry after PCA-projection so as to balance variance over the kept directions. Works from~\cite{Jegou2010} and~\cite{FROSH17} use a random rotation while the rotation can also be learned~(\cite{ITQ2013, IsoHash2012}) to that purpose. This led to the development of variant online techniques such as Online Sketching Hashing (OSH)~(\cite{LengWC0L15}), FasteR Online Hashing (FROSH)~(\cite{FROSH17}), UnifDiag~(\cite{MorvanSGA18}) which are deployable in the streaming setting when large high-dimensional data should be processed with limited memory. These are currently the state-of-the-art of online unsupervised hashing methods for learning on the fly similarity-preserving binary embeddings.    
Nevertheless, even if one can now use these unsupervised online methods for processing high-dimensional streams of data, there is still no theoretical justification that equalizing the variance or in other words choosing directions with isotropic variance leads to optimal results.

\textbf{Contributions:} The main contribution of this paper is to  bring theoretical guarantees to several state of the art quantization-based hashing techniques, by formally proving the need for rotation when the methods rely on PCA-like preprocessing.
We also introduce experiments accompanying the theoretical results and extend ones comparing existing online unsupervised quantization-based hashing methods.

\section{Related work} 
Two paradigms stand for building hash functions: data-independent~(\cite{Andoni2008NHA, SKLSH09, KLSH11}) and data-dependent methods~(\cite{survey18}). 
For the latter category, the binary embeddings are learned from the training set and are known to work better. 
Among this branch of methods, unsupervised methods~(\cite{SH08, Liu11hashingwith, ITQ2013,IsoHash2012, Lee2012SphericalH, DisGraphHashingNIPS2014, CBEICML14, Raziperchikolaei16}) design hash codes preserving distances in the original space while (semi-)supervised ones attempt to keep label similarity~(\cite{SHL12, supervisedHK12}). 
Work from~(\cite{survey18}) proposes an extensive survey of these methods and makes further distinctions to finally show that quantization-based techniques are superior in terms of search accuracy. We can cite: ITerative Quantization (ITQ)~(\cite{ITQ2013}), Isotropic Hashing (IsoHash)~(\cite{IsoHash2012}), Cartesian K-means~(\cite{cartKmeans13}), deep-learning-based methods~(\cite{LaiPLY15, DeepLiong15,Do2016LearningTH}). 
Specifically, when the dataset is too large to fit into memory, distributed hashing can be used~(\cite{HDDleng15}) or
online hashing techniques~(\cite{Huang2013OH, LengWC0L15, AOH15,MIHash17,MorvanSGA18}) can process the data in only one pass, as a continuous stream, and compute binary hash codes as new data is seen. This latter area has attracted lots of interest in the past few years. 
Online Hashing (OKH)~(\cite{Huang2013OH}) learns the hash functions from a stream of similarity-labeled pair of data with a ``Passive-Aggressive" method. Supervised MIHash algorithm~(\cite{MIHash17}) also uses similarity labels between pairs of data and consider Mutual Information for computing the binary embeddings.
On the unsupervised side, in Online Sketching Hashing (OSH)~(\cite{LengWC0L15}), the binary embeddings are learned from a maintained sketch of the dataset with a smaller size which preserves the property of interest. UnifDiag, the recent approach from~\cite{MorvanSGA18}, learns the hash codes from an online estimated principal space and the balancing rotation as a sequence of Givens rotation whose coefficients are computed so as to equalize diagonal coefficients of the PCA-projected data covariance. In the same work, it has also been shown how to adapt IsoHash~(\cite{IsoHash2012}), which is not originally an online technique, to the streaming setting.

\section{Notations and problem statement}

In the sequel, for a given integer $a$, $[a] = \{1, \ldots, a \}$. Let us define the $\sign$ function s.t. for any real $x$, $\operatorname{sign}(x) = 1$ if $x \geq 0$ and $-1$ otherwise. We also use this notation for the same function applied component-wise on coefficients of vectors. $||.||_{F}$ denotes the Froebenius norm. $\Tr(.)$ stands for the Trace application.
$dist_H(.,.)$ returns the Hamming distance.
For any real $a$, $\diag_c(a)$ returns a $c \times c$-diagonal matrix whose diagonal coefficients are all equal to $a$, while for any matrix $M$, $\diag(M)$ returns a diagonal matrix with the same diagonal as $M$. $\mathcal{O}(c)$ is the set of all orthogonal matrices in $\mathbb{R}^{c \times c}$, i.e. $\mathcal{O}(c) = \{ Q \ | \ Q^T Q = Q Q^T = I_c \}$ where $I_c$ is the $c\times c$ Identity matrix. For any matrix $M$, $\Sigma_M = M M^T$. For any vector $z$, $z^{(i)}$ denotes its $i^{th}$ entry.

Let be a stream of $n$ zero-centered data points $\{ x_t \in \mathbb{R}^d\}_{1 \leq t \leq n}$.
The considered hashing methods aim at obtaining binary codes $b_t = \operatorname{sign}(\tilde{W} x_t) \in \{-1, 1\}^{c}$ for $t \in [n]$, where: $c$ denotes the code length, $c \ll d$, and $\tilde{W} \in \mathbb{R}^{c \times d}$ is the dimensionality reduction operator.
In other words, for each bit $k \in [c]$, the hashing function is defined as 
$h_k(x_t) = \operatorname{sign}(\tilde{w}^T_{k} x_t)$ 
where $\tilde{w}_{k}$ are column vectors of hyperplane coefficients.
So $\tilde{w}_{k}^T$ is a row of $\tilde{W} \in \mathbb{R}^{c \times d}$ for each $k$. 
In the framework of hypercubic quantization hashing functions, $\tilde{W} = R W$ where $W$ is the linear dimension reduction embedding applied to data and $R$ is a suitable $c \times c$ orthogonal matrix. 
Typically, one can take $W$ as the matrix whose row vectors $w_{k}^T$ are the $c$ first principal components of the covariance matrix $\Sigma_{X}$ where $X = [x_1, \ldots, x_n] \in R^{d \times n}$. 
Please note that this problem statement includes offline and online methods: $W$ can be either the PCA or a tracked principal subspace as new data is seen. What is more specific to the methods is the way of learning the appropriate orthogonal matrix\footnote{In the sequel, the terms orthogonal matrix or rotation are used equivalently.} for rotating the data previously projected onto this principal subspace.
After application of the (possibly online estimated) PCA algorithm, the stream becomes $\{ v_t \in \mathbb{R}^c\}_{1 \leq t \leq n}$ s.t. $v_t = W x_t$.
Then, for all $t \in [n]$, $y_t = Rv_t$ and $b_t = \sign(y_t)$. 
In the sequel, depending on the context, we will be considering either one data point $x_t$, its initial projection $v_t = Wx_t$, its rotated projection $y_t$ and binary sketch $b_t$, or the whole associated sets $X$, $V = WX$, $Y = RV$ and $B = \sign(Y) = [b_1, \ldots, b_n] \in \{-1, 1\}^{c \times n}$.

\subsection{ITerative Quantization (ITQ)~(\cite{ITQ2013})}

For ITQ, $R$ is the solution of an orthogonal Procustes problem which consists in minimizing quantization error $Q(B,R)$ of mapping resulting data to the vertices of the $2^c$ hypercube:
\begin{equation}
Q(B,R) = || B - \tilde{W}X||^{2}_{F} = || B - RWX||^{2}_{F} = || B - RV||^{2}_{F}.
\end{equation} 
$R$ is initially some random orthogonal matrix. Then, iteratively, ITQ alternatively computes $B = \sign(RWX)$ after freezing $R$, and optimizes $R$ according to $B$. In the latter step, $R = \tilde{S} S^T$ by defining $S, \Omega , \tilde{S}^T = \SVD(B^TV)$ the Singular Values Decomposition (SVD) of $B^TV$. 
Hence, ITQ's goal is to map the values of projected data to their component-wise sign.

\subsection{Methods based on uniformizing diagonal coefficients of covariance matrix}

Let $\sigma^2_1, ..., \sigma^2_c$ be the diagonal coefficients of $\Sigma_{V}$, hence $\sigma^2_1 \geq ... \geq \sigma^2_c$.
The methods described thereafter look for a matrix $R$ balancing variance over the $c$ directions, i.e. equalizing the diagonal coefficients of $\Sigma_Y$ to the same value $\tau = \operatorname{Tr}(\Sigma_V) / c$. 

\subsubsection{Isotropic Hashing (IsoHash)~(\cite{IsoHash2012})}
Let us define, for any real $a$, $\mathcal{T}(a)$, the set of all $c \times c$ matrices with diagonal coefficients equal to $a$: 
$$\mathcal{T}(a) = \{T \in \mathbb{R}^{c \times c} \ | \ \diag(T) = \diag_c(a) \}$$
and $\mathcal{M}(\Sigma_V) = \{Q \Sigma_V Q^T \ | \ Q \in \mathcal{O}(c) \}$. Then, in IsoHash, the chosen way for determining $R$ is to solve the following optimization problem: 
$$R \in \underset{ Q: \ T \in \mathcal{T}(\tau), \ Z \in \mathcal{M}(\Sigma_V)}{ \operatorname{argmin} } || T - Z||_F.$$

One of the proposed methods by IsoHash to solve this problem (\textit{Gradient Flow}) is to reformulate it as the following:
$$
R \in \underset{Q \in \mathcal{O}(c)}{\operatorname{argmin}} \frac{1}{2} || \diag(Q \Sigma_V Q^T) - \diag(\tau)||^2_F.
$$
Then, the rotation results from a gradient descent converging to the intersection between the set of orthogonal matrices and the set of transfer matrices making $\Sigma_Y$ diagonal. 

\subsubsection{UnifDiag Hashing~(\cite{MorvanSGA18})}

With UnifDiag, $\Sigma_{V}$ is dynamically computed as new data is seen while updating $W$ with OPAST algorithm~(\cite{OPAST2000}).
$R$ is defined as a product of $c-1$ Givens rotations $G(i, j, \theta)$.
\begin{definition} \label{def:givens}
A Givens rotation $G(i,j, \theta)$ is a matrix of the form: 
\begin{center}
$G(i,j, \theta) = \begin{bmatrix}
1 & \cdots & 0 & \cdots & 0 & \cdots & 0 \\
\vdots & \ddots &\vdots & & \vdots & & \vdots \\ 
0 & \cdots & c & \cdots & -s & \cdots & 0 \\
\vdots & & \vdots & \ddots & \vdots & & \vdots \\ 
0 & \cdots & s & \cdots & c & \cdots & 0 \\
\vdots & & \vdots & & \vdots & \ddots & \vdots \\ 
0 & \cdots & 0 & \cdots & 0 & \cdots &1
\end{bmatrix}$ 
\end{center}
where $i > j$, $c = \cos(\theta)$ and $s = \sin(\theta)$; 
$\forall k \neq i, \, j, \ g_{k,k} =1$; $g_{i,i} = g_{j,j} = c$, $g_{j,i} = -s$ and $g_{i,j} = s$. All remaining coefficients are set to $0$.
\end{definition} 
Some Givens rotations are iteratively applied left and right to $\Sigma_V$ during the iterative Jacobi eigenvalue algorithm for matrix diagonalization~(\cite{Golub2000}). For $r \in [c-1]$, given $i_r, j_r, \theta_r$,
\begin{align}
{\left(\Sigma_{Y}\right)}_r & \leftarrow G(i_r,j_r, \theta_{r}) \ {\left(\Sigma_{Y}\right)}_{r-1} \ G(i_r,j_r, \theta_{r})^T \label{eq:updateCov} \\
R_{r} &\leftarrow R_{r-1} \ G(i_r,j_r, \theta_r)^T, 
\end{align}
where ${\left(\Sigma_{Y}\right)}_0 = \Sigma_V$, $R_0 = I_c$.
At each step $r$, $i_r$ and $j_r$ are chosen to be the indices of the current smallest and largest diagonal coefficients of ${\left(\Sigma_{Y}\right)}_{r-1}$.$\theta_r$ is computed accordingly so that at the end of step $r$, $r$ diagonal coefficients of ${\left(\Sigma_{Y}\right)}_{r}$ are equal to $\tau$. 

\section{Theoretical justification of optimality of a rotation $\mathbf{R}$ for hypercubic quantization hashing}
We deliver in this Section three results on the theoretical justification of applying $R$ after PCA projection in terms of efficiency of the binary sketches.
We prove in Section~\ref{sec:th_unifdiag}, assuming some distribution on data, 1) the optimality of choosing $R$ as a rotation uniformizing the diagonal of the covariance matrix and 2) we provide some lower bound on the probability of getting different binary sketches for two data points initially close to each other.
3) Then we propose an extension to the case with no assumption on data distribution but some on matrix $R$ in Section~\ref{sec:th_random}. Given an upper bound on the distance between two data points in the initial space, we give a lower bound on the number of bits in common in their binary sketches.
\subsection{When rotation $\mathbf{R}$ uniformizes the diagonal of the covariance matrix} \label{sec:th_unifdiag}

In this section, we assume  $R \in \mathbb{R}^{c \times c}$ is a rotation matrix ($RR^T = R^TR = I_c)$.

\begin{hypothesis}[H1] \label{hyp:H1} We assume:
$\forall t \in [n]$, $(v_t^{(1)}, v_t^{(2)}, \ldots, v_t^{(c)})^T  \sim \mathcal{N}( \textbf{0}, \Sigma^{th}_V)$ s.t. diagonal coefficients of $\Sigma^{th}_V$ are $({\sigma^{th}_1}^2, \ldots, {\sigma^{th}_c}^2)$.
In particular, $\forall t \in [n], \forall i \in [c]$, $v_t^{(i)} \sim \mathcal{N}(0, {\sigma^{th}_i}^2)$. 
$\Sigma^{th}_V$ is not necessary diagonal.
\end{hypothesis} 
Before introducing Th.~\ref{thm:UnifDiagGauss1}, please note that two neighboring data points can have very dissimilar binary codes if their projections on PCA have many coefficients near zero, not on the same side of the hyperplanes delimiting the orthants. Indeed, in this case, the $\sign$ function will attribute opposite bits. Therefore, after dimensionality reduction of the data points in the original space, a good hashing method tends to keep the coefficients of the projected data away from zero, i.e. away from the orthant. The challenge is then to determine how to move these data points away from the orthants.

\begin{theorem} \label{thm:UnifDiagGauss1}
Assume $\{ x_t \in \mathbb{R}^d\}_{1 \leq t \leq n}$ is a stream of $n$ zero-centered data points following Hypothesis~\ref{hyp:H1}.
Then, choosing $R$ so that it uniformizes the diagonal of the covariance matrix $\Sigma^{th}_Y$ is equivalent to minimizing some upper bound on the probability that data points are closed to an hyperplane delimiting an orthant.
\end{theorem}

\begin{proof}
Let $\epsilon > 0$.
For $i \in [c]$, $t \in [n]$, let $p_i^\epsilon$ be the probability (independent of $t$) that $y_t^{(i)} = (R v_t)^{(i)}$ is closer than $\epsilon$ from the orthant:
\begin{align*}
p^\epsilon_i 
&= \mathbb{P}[ \ |y_t^{(i)}| < \epsilon \ ] 
= \int^{\epsilon}_{-\epsilon} \frac{1}{\sqrt{2 \pi (R\Sigma^{th}_V R^T)_{ii}}} e^{ \frac{-s^2}{2 (R\Sigma^{th}_V R^T)_{ii}} } ds \\
&= \frac{2 \epsilon}{ \sqrt{2 \pi (R\Sigma^{th}_V R^T)_{ii}}} + \underbrace{o(\epsilon^2)}_{\leq 0} \leq \frac{2 \epsilon}{ \sqrt{2 \pi (R\Sigma^{th}_V R^T)_{ii}}}.
\end{align*}
Hence, 
\begin{align}
\underset{R}{\operatorname{min}} & \ \mathbb{P}[ \ \underset{i \in [c]}{\bigcup}  \left( |y_t^{(i)}| < \epsilon \right) ]
\leq \underset{R}{\operatorname{min}} \underset{i \in [c]}{\sum} p^\epsilon_i \label{eq:ineq_indep_dim} \\
&\leq \underset{R}{\operatorname{min}} \left( \frac{2 \epsilon}{ \sqrt{2 \pi}} \underset{i \in [c]}{\sum}\frac{1}{\sqrt{(R\Sigma^{th}_V R^T)_{ii}}} \right). \notag 
\end{align}
Now, let us define $R^* \in \underset{R}{\operatorname{argmin}} \left( \frac{2 \epsilon}{ \sqrt{2 \pi}} \underset{i \in [c]}{\sum}\frac{1}{\sqrt{(R\Sigma^{th}_V R^T)_{ii}}} \right)$.
We denote for all $i \in [c]$, $\gamma^R_i = \sqrt{(R\Sigma^{th}_V R^T)_{ii}}$ and
$\gamma^R = ( \gamma^R_1, \gamma^R_2, \ldots, \gamma^R_c)^T$. 
Then, Cauchy-Schwartz inequality gives:
$\left\langle \textbf{1}, \gamma^R \right\rangle^2 \leq \left\langle \textbf{1}, \textbf{1} \right\rangle \left\langle \gamma^R, \gamma^R \right\rangle$ \\
i.e. $\left( \underset{i \in [c]}{\sum} \gamma^R_i \right)^2 \leq c \ . \underset{i \in [c]}{\sum} (\gamma^R_i)^2$ which rewrites:
\begin{equation}
\left(  \underset{i \in [c]}{\sum} \gamma^R_i \right)^{-1} \geq c^{-\frac{1}{2}} \ . \left(  \underset{i \in [c]}{\sum} (\gamma^R_i)^2 \right)^{- \frac{1}{2}} \label{eq:cs1bis}
\end{equation}
Besides, $c^2 = \left\langle (\gamma^R)^{ -\frac{1}{2}}, (\gamma^R)^{ \frac{1}{2}} \right\rangle^2 \leq \underset{i \in [c]}{\sum} (\gamma_i^R)^{-1} \ . \underset{i \in [c]}{\sum} \gamma^R_i$ rewrites:
\begin{equation}
\underset{i \in [c]}{\sum} (\gamma_i^R)^{-1} \geq c^2 \ . \left(  \underset{i \in [c]}{\sum} \gamma^R_i \right)^{- 1} \label{eq:cs2bis}
\end{equation}
$R$ is a rotation, $\Tr(R\Sigma^{th}_V R^T) = \Tr(\Sigma^{th}_V) = \underset{i \in [c]}{\sum} {(\sigma^{th}_i)}^2$. 
So, $\underset{i \in [c]}{\sum}(\gamma^R_i)^2 = \underset{i \in [c]}{\sum}(R\Sigma^{th}_V R^T)_{ii} = \Tr(\Sigma^{th}_V)$ is a constant of $R$. 
Then, Eq.~\ref{eq:cs1bis} and~\ref{eq:cs2bis} give: 
\begin{equation}
\underset{i \in [c]}{\sum} (\gamma^R_i)^{-1} \geq c^{ \frac{3}{2}} \ . \left(  \underset{i \in [c]}{\sum} (\gamma^R_i)^2 \right)^{- \frac{1}{2}} =  c^{ \frac{3}{2}} C^{- \frac{1}{2}}
\end{equation}
Minimal value of $\underset{i \in [c]}{\sum} (\gamma^R_i)^{-1}$ is reached if and only if equality holds in the Cauchy-Schwartz inequalities i.e. if and only if, for all $i \in [c]$, $\gamma^R_i$ are equal. Hence, for all $i \in [c]$, $\gamma^{R^*}_i = \gamma^{R^*}_1$. Conversely, if $\gamma^{R^*}_i = \gamma^{R^*}_1$  for all $i \in [c]$, then $R^* \in \underset{R}{\operatorname{argmin}} \left( \frac{2 \epsilon}{ \sqrt{2 \pi}} \underset{i \in [c]}{\sum}\frac{1}{\sqrt{(R\Sigma^{th}_V R^T)_{ii}}} \right)$.
\end{proof}

If moreover $\Sigma^{th}_V$ is diagonal, inequality~\ref{eq:ineq_indep_dim} becomes an equality. Thus, choosing $R$ so that it uniformizes the diagonal of the covariance matrix $\Sigma^{th}_Y$ is exactly equivalent to minimizing the probability that data points are closed to an hyperplane delimiting an orthant plus $o(\epsilon^2)$, where $\epsilon$ is the distance to the orthant.

Please note that in practice the algorithm uniformizes the diagonal coefficients of the \emph{empirical} covariance matrix $\Sigma_V$. This still makes sense because $\frac{1}{n}\Sigma_V$ is a consistent estimator of the theoretical covariance matrix $\Sigma^{th}_V$.

Now we can bound the probability of getting dissimilar sketches for data points close to each other, as stated in Th.~\ref{thm:UnifDiagGauss2} below:
\begin{theorem} \label{thm:UnifDiagGauss2}
Let $x_{t_1} \in \mathbb{R}^{d}$ and $x_{t_2} \in \mathbb{R}^d$ be two data points following Hypothesis~\ref{hyp:H1}, $\epsilon > 0$ so that $|| x_{t_1} -  x_{t_2}||_2 \leq \epsilon$ and $b_{t_1} \in \{ -1, 1\}^{c}$, $b_{t_2} \in \{ -1, 1\}^c$ with $b_{t_i} = \sign(RWx_{t_i})$ for $i \in \{1, 2 \}$.
Then, the probability of getting dissimilar binary sketches is upper bounded as follows:
\begin{equation}
\mathbb{P}[ \ dist_H(b_{t_1}, b_{t_2}) > 0] \leq 2 \epsilon \sqrt{\frac{2}{\pi}} c^{ \frac{3}{2}} \left( \Tr(\Sigma^{th}_V)\right)^{- \frac{1}{2}}. 
\end{equation}
\end{theorem} 

\begin{proof}
As PCA performs a projection, one has: $|| v_{t_1} - v_{t_2}||_2 \leq  ||x_{t_1} - x_{t_2} ||_2 \leq \epsilon$.
Then, $|| y_{t_1} - y_{t_2} ||_2 \leq \epsilon$ since $R$ preserves the norm as a rotation. Thus, in particular, for all $i \in [c]$, $|y_{t_1}^{(i)} - y_{t_2}^{(i)}| \leq \epsilon$.
Then, 
\begin{align}
\mathbb{P}[ \ dist_H&(b_{t_1}, b_{t_2}) > 0] = \mathbb{P} [ \underset{i \in [c]}{\bigcup} \left( y_{t_1}^{(i)} y_{t_2}^{(i)} < 0 \right)] \notag \\
& \leq \mathbb{P} [ \underset{i \in [c]}{\bigcup} \left( |y_{t_1}^{(i)}| < \epsilon \ \cap \ |y_{t_2}^{(i)}| < \epsilon \right)] \label{eq:proba_eq1}
\end{align}
because $\left( |y_{t_1}^{(i)}| \geq \epsilon \cup |y_{t_2}^{(i)}| \geq \epsilon \right) \implies \left( y_{t_1}^{(i)} y_{t_2}^{(i)} \geq 0 \right)$. 
Moreover,
\begin{align*}
\mathbb{P} [ \underset{i \in [c]}{\bigcup} \left( |y_{t_1}^{(i)}| < \epsilon \ \cap \ |y_{t_2}^{(i)}| < \epsilon \right)] \\
\leq \mathbb{P} [ \underset{i \in [c]}{\bigcup} |y_{t_1}^{(i)}| < \epsilon ]  + \mathbb{P} [ \underset{i \in [c]}{\bigcup} |y_{t_2}^{(i)}| < \epsilon ] 
\end{align*}
Since $y_{t_1}$ and $y_{t_2}$ have same distribution and using Th.~\ref{thm:UnifDiagGauss1},
\begin{align}
\mathbb{P} &[ \underset{i \in [c]}{ \bigcup} |y_{t_1}^{(i)}| < \epsilon ] + \mathbb{P} [ \underset{i \in [c]}{\bigcup} |y_{t_2}^{(i)}| < \epsilon ] 
= 2 \mathbb{P} [ \underset{i \in [c]}{\bigcup} |y_{t_1}^{(i)}| < \epsilon ] \notag 
\\
&\leq 2 \underset{i \in [c]}{\sum} \mathbb{P} [ \ |y_{t_1}^{(i)}| < \epsilon ] = 2 \epsilon \sqrt{\frac{2}{\pi}} c^{ \frac{3}{2}} \left( \Tr(\Sigma^{th}_V) \right)^{- \frac{1}{2}} \notag 
\end{align}
The result follows.
\end{proof}

\subsection{When matrix $\mathbf{R}$ is random}  \label{sec:th_random}

We now show that if we assume $R$ is a random matrix that resembles a random rotation (i.e. we do not optimize $R$ and furthermore $R$ is not a valid rotation matrix) we can still get strong guarantees regarding the quality of the hashing mechanism (even though not as strong as in the case when $R$ is optimized).
We will assume that $R$ is a Gaussian matrix with entries taken independently at random from $\mathcal{N}(0,1)$. We further make some assumptions regarding the input data and the quality of the PCA projection mechanism. 

\begin{hypothesis}[H2] \label{hyp:two_assumptions}
We assume that the variance of the norm of vectors $x_{t}$ is upper-bounded, namely we assume that $l(1-\delta) \leq \|x_{t}\|_2 \leq l(1+\delta)$ for some fixed $l,\delta > 0$. We also assume that the PCA projection encoded by matrix $W$ satisfies for every $t$: $\|Wx_{t}-x_{t}\|_2 \leq \epsilon \|x_{t}\|_2$ for some $\epsilon > 0$.
The latter condition means that the fraction of the $L_{2}$-norm of the vector lost by performing a PCA projection is upper-bounded (better quality PCA projections are characterized by smaller values of $\epsilon$).
\end{hypothesis}

We now state the main theoretical result of this section:

\begin{theorem}
Let $x_{t_{1}},x_{t_{2}} \in \mathbb{R}^{d}$ be two data points following Hypothesis \ref{hyp:two_assumptions} for some constants $l,\delta,\epsilon > 0$. Then for every $\rho>0$ the following holds: if $\|x_{t_{1}}-x_{t_{2}}\|_2 \leq \rho$ and $b_{t_{1}} \in \{-1,1\}^{c}$ and $b_{t_{2}} \in \{-1,1\}^{c}$ with $b_{t_{1}} = \sign(RWx_{t_{1}})$ and $b_{t_{2}}=\sign(RWx_{t_{2}})$ then for any $\eta>0$ with probability at least $1-e^{-\frac{\eta^{2}c}{2}}$ the number of bits in common between $b_{t_{1}}$ and $b_{t_{2}}$ is lower bounded by $q$ satisfying
\begin{equation}
q = c (1-\eta-\frac{1}{\pi} \arccos(A) )
\end{equation}
with 
\begin{align}
A = \frac{(1- \epsilon)^2 (1 - \delta)^2}{ (1 + \delta)^2} - \frac{\rho^2}{2l^2 (1 + \delta)^2} - \frac{\rho \epsilon}{ l (1 + \delta)} - 2 \epsilon^2
\end{align}
\end{theorem}

\begin{proof}
Note that by Hypothesis \ref{hyp:two_assumptions} and triangle inequality, we know that $v_{t_{1}} = Wx_{t_{1}}$ and $v_{t_{2}} = Wx_{t_{2}}$
satisfy 
\begin{align}
\|v_{t_{1}}-v_{t_{2}}\|_2 
&= \|v_{t_{1}} - x_{t_1} - (v_{t_{2}} - x_{t_2}) + x_{t_1} - x_{t_2} \|_2 \notag \\
&\leq \rho + 2\epsilon(1+\delta)l. \label{eq:rho}
\end{align}
Now consider a particular entry $j \in \{1, \ldots,c\}$ of two binary codes: $b_{t_{1}}$ and $b_{t_{2}}$. Note that
$b^{(i)}_{t_{1}} = \sign((r^{i})^{\top}v_{t_{1}})$ and $b^{(i)}_{t_{2}} = \sign((r^{i})^{\top}v_{t_{2}})$, where $r^{i}$ stands for the transpose of the $i^{th}$ row of $R$.
We have: 
\begin{equation}
(r^{i})^{\top}v_{t_{1}} = (r^{i}_{\mathrm{proj}}+r^{i}_{\mathrm{ort}})^{\top}v_{t_{1}}= (r^{i}_{\mathrm{proj}})^{\top}v_{t_{1}},
\end{equation}
where $r^{i}_{\mathrm{proj}}$ stands for the orthogonal projection of the vector $r^{i}$ into a $2$-dimensional linear space spanned by $\{v_{t_{1}},v_{t_{2}}\}$ and 
$r^{i}_{\mathrm{ort}}$ denoted the part of the vector $r^{i}$ that is orthogonal to that subspace.

Similarly, we obtain:
\begin{equation}
(r^{i})^{\top}v_{t_{2}} = (r^{i}_{\mathrm{proj}}+r^{i}_{\mathrm{ort}})^{\top}v_{t_{2}}= (r^{i}_{\mathrm{proj}})^{\top}v_{t_{2}}.
\end{equation}

Therefore we have: $b^{(i)}_{t_{1}} = \sign((r^{i}_{\mathrm{proj}})^{\top}v_{t_{1}})$ and 
$b^{(i)}_{t_{2}} = \sign((r^{i}_{\mathrm{proj}})^{\top}v_{t_{2}})$.

Denote by $\mathcal{L}$ the subset of the two-dimensional subspace spanned by $\{v_{t_{1}},v_{t_{2}}\}$ that consists of these vectors $v \in \mathbb{R}^{c}$ that satisfy: 
$(v^{\top}v_{t_{1}})(v^{\top}v_{t_{2}}) > 0$.

Now note that the projection $r^{i}_{\mathrm{proj}}$ of the Gaussian vector $r^{i}$ into two-dimensional deterministic linear subspace spanned by $\{v_{t_{1}},v_{t_{2}}\}$ is still Gaussian. Furthermore, Gaussian vectors satisfy the isotropic property. The probability that $b^{(i)}_{t_{1}}$ and 
$b^{(i)}_{t_{2}}$ are the same is exactly the probability that $r^{i}_{\mathrm{proj}} \in \mathcal{L}$. From the fact that $r^{i}_{\mathrm{proj}}$ is Gaussian and the isotropic property of Gaussian vectors~(\cite{Charikar2002}) we conclude that:
\begin{equation}
\mathbb{P}[ b^{(i)}_{t_{1}} = b^{(i)}_{t_{2}} ] = \mathbb{P}[r^{i}_{\mathrm{proj}} \in \mathcal{L}] = 1-\frac{\theta_{v_{t_{1}},v_{t_{2}}}}{\pi},
\end{equation}
where $\theta_{v_{t_{1}},v_{t_{2}}}$ stands for an angle between $v_{t_{1}}$ and $v_{t_{2}}$.
Besides, we have: 
\begin{equation}
\cos(\theta_{v_{t_{1}},v_{t_{2}}}) = \frac{ \|v_{t_{1}}\|^{2}_2 +\|v_{t_{2}}\|^{2}_2 - \|v_{t_{1}}-v_{t_{2}}\|^{2}_2}{ 2\|v_{t_{1}}\|_2 \|v_{t_{2}}\|_2 }
\end{equation}
Note that by property of the PCA, one has: 
$||v_{t_1}||_2 \leq ||x_{t_1}||_2$ and $||v_{t_2}||_2 \leq ||x_{t_2}||_2$. 
Moreover, triangle inequality gives:
$\| v_{t_1} \|_2 \geq \|x_{t_1} \|_2 - \| v_{t_1} - x_{t_1} \|_2$.
Thus, $\| v_{t_1} \|_2 \geq (1 - \epsilon)(1 - \delta)l$ by using Hypothesis \ref{hyp:two_assumptions}. 
Similarly, $\| v_{t_2} \|_2 \geq (1 - \epsilon)(1 - \delta)l$.
From the above and by using Eq.~\ref{eq:rho} we get:
\begin{align}
\cos&(\theta_{v_{t_{1}},v_{t_{2}}}) \geq \frac{ 2 (1 - \epsilon)^2 (1 - \delta)^2 l^2 - (\rho + 2 \epsilon (1 + \delta) l)^2}{ 2 l^2(1 + \delta)^2} \notag \\
&\geq \frac{(1- \epsilon)^2 (1 - \delta)^2}{ (1 + \delta)^2} - \frac{\rho^2}{2l^2 (1 + \delta)^2} - \frac{\rho \epsilon}{ l (1 + \delta)} - 2 \epsilon^2
\end{align} %
Hence,
\begin{align} \label{eq:upper_bound_theta}
\theta_{v_{t_{1}},v_{t_{2}}} \leq \arccos &\left(  \frac{(1- \epsilon)^2 (1 - \delta)^2}{ (1 + \delta)^2} - \frac{\rho^2}{2l^2 (1 + \delta)^2} \right. \notag \\
& \left. - \frac{\rho \epsilon}{ l (1 + \delta)} - 2 \epsilon^2 \right).
\end{align}%
Now denote by $X_{i}$ a random variable that is $+1$ if 
$b^{(i)}_{t_{1}}$ and $b^{(i)}_{t_{2}}$ are the same and is $0$ otherwise. Note that the probability that $X_{i}$ is nonzero is exactly 
\begin{equation} \label{eq:proba_p}
p=1-\frac{\theta_{v_{t_{1}},v_{t_{2}}}}{\pi}.
\end{equation}
Note also that different $X_{i}$ are independent since different rows $r^{i}$ of the matrix $R$ are independent.
If we denote by $X$ a random variable defined as:
\begin{equation}
X = X_{1}+ \ldots +X_{c},
\end{equation}
then note that $\mathbb{E}[X]=pc$.
This random variable $X$ counts the number of entries of $b_{t_{1}}$ and $b_{t_{2}}$ that are the same.
Now, using standard concentration results such as Azuma's inequality, we can conclude that for all $a > 0$ :
\begin{equation}
\mathbb{P}[X-pc < -a] \leq e^{-\frac{a^{2}}{2c}}.
\end{equation}
In particular, for $a = \eta c$,
\begin{equation}
\mathbb{P}[X-pc < -\eta c] \leq e^{-\frac{\eta^{2}c}{2}}.
\end{equation}
Then with probability at least $1 - e^{-\frac{\eta^{2}c}{2}}$, $X \geq c( p - \eta)$.
Putting the formula on $p$ from Eq.~\ref{eq:proba_p} and the obtained upper bound on $\theta_{v_{t_{1}},v_{t_{2}}}$ from Eq.~\ref{eq:upper_bound_theta}, we complete the proof.

\end{proof}

\section{Experiments}

Experiments have been carried out on a single processor machine (Intel Core i7-5600U CPU @ 2.60GHz, 4 hyper-threads) with 16GB RAM and implemented in Python (code will be available after publication).

\subsection{Comparison of Hypercubic Quantization Hashing methods for the Nearest Neighbors search task} 

We propose first further experiments in comparison with the recent approach from ~\cite{MorvanSGA18} for more code lengths in batch and online setting. 
Quality of hashing is assessed here on the Nearest Neighbors (NN) search task: retrieved results of the NN search task performed on the $c$-bits codes of hashed data points are compared with the true nearest neighbors induced by the Euclidean distance applied on the initial $d$-dimensional real-valued descriptors.
Mean Average Precision (MAP), commonly used in Information Retrieval tasks, measures then the accuracy of the result by taking into account the number of well-retrieved nearest neighbors and their rank.
Based on the MAP criteria, two types of experiments are presented: 1) UnifDiag algorithm is compared with batch-based methods to show that the streaming constraint does not make loose too much accuracy in NN search results. 2) In the online context, the algorithm is set against to existing online methods in order to exhibit its efficiency.
In both cases, tests were conducted on two datasets: CIFAR-10 (\url{http://www.cs.toronto.edu/~kriz/cifar.html}) and GIST1M (\url{http://corpus-texmex.irisa.fr/}).
CIFAR-10 (CIFAR) contains $60000$ $32 \times 32$ color images equally divided into 10 classes. 960-D GIST descriptors were extracted from those data. GIST1M (GIST) contains $1$ million $960$-D GIST descriptors, from which $60000$ instances were randomly chosen from the first half of the learning set. 
To perform the NN search task, $1000$ queries are randomly sampled and the $59 000$ remaining data points are used as training set. Then, the sets of neighbors and non-neighbors of the queries are determined by a nominal threshold which is arbitrarily chosen to be the average distance to the $600^{th}$ nearest neighbor in the training set ($1\%$ of each dataset). After binary hashing of the query and training sets, MAP at $2000$ is computed over the first $2000$ retrieved nearest neighbors from the training set according to sorted values of Hamming distance between the binary codes. Indeed, we are obviously interested in nearest neighbors returned first, so $MAP@2000$ is enough.
Results are averaged over $5$ random training/test partitions.   

\subsubsection{Comparison with batch-based methods}

In the offline setting, the hashing function is learned over the whole training set and the final MAP is printed for different $c$ value.
Fig.~\ref{fig:batch_comparison} shows MAP results for UnifDiag against unsupervised batch-based methods: ITQ ($K = 50$) and IsoHash (original version preceded by a PCA projection), for an evaluation after having seen the whole training set. The online estimation of the principal subspace via OPAST instead of the classical PCA does not lead to a loss of accuracy, since UnifDiag reaches similar performances to batch-based methods for every code lengths.

\begin{figure}
\centering
\begin{subfigure}{0.44\textwidth}
  \centering
  \includegraphics[width = 0.87654\linewidth]{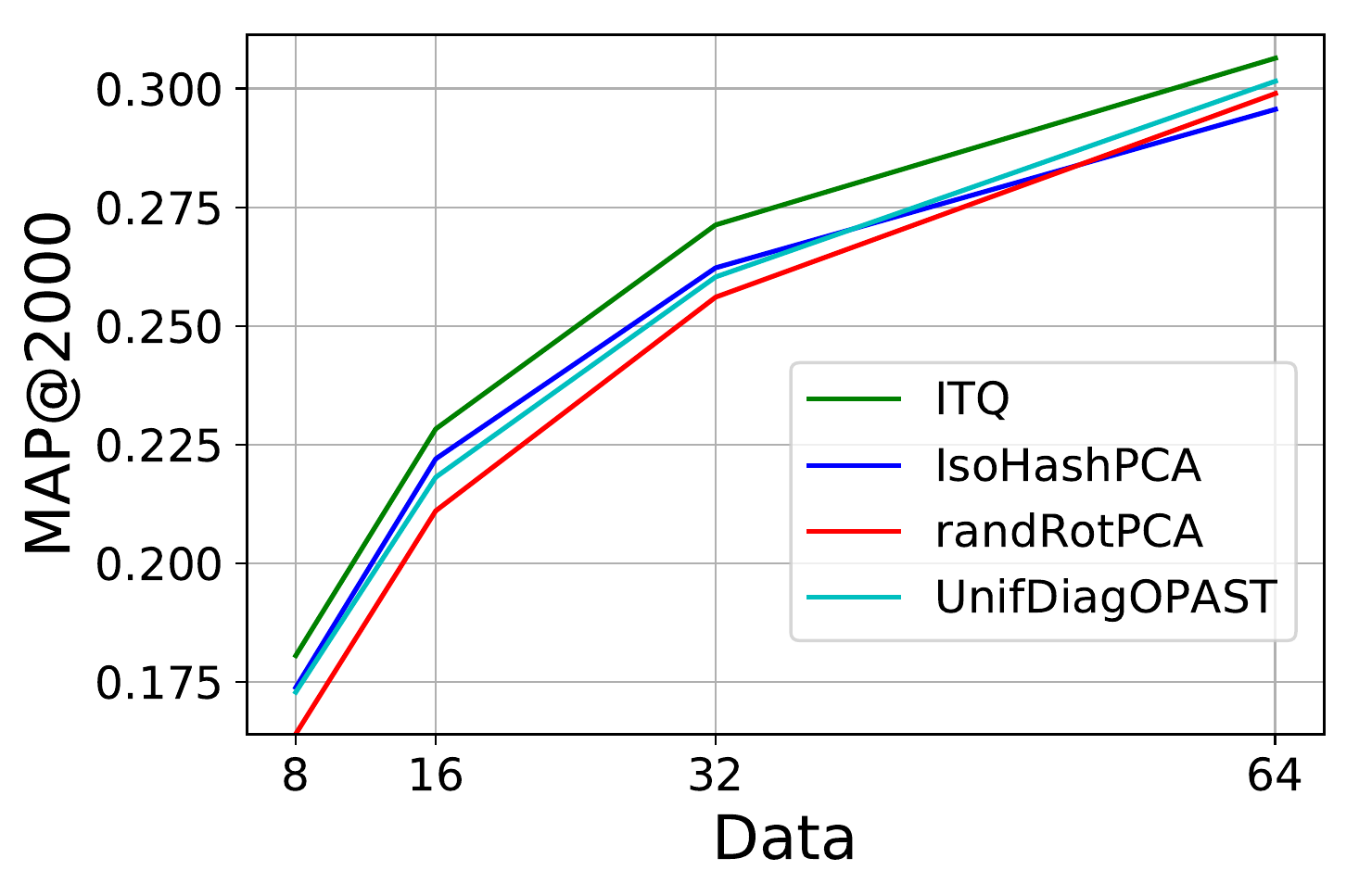}
  \caption{CIFAR}
  \label{subfig:batch_CIFAR}
\end{subfigure}%
\begin{subfigure}{0.44\textwidth}
  \centering
  \includegraphics[width = 0.87654\linewidth]{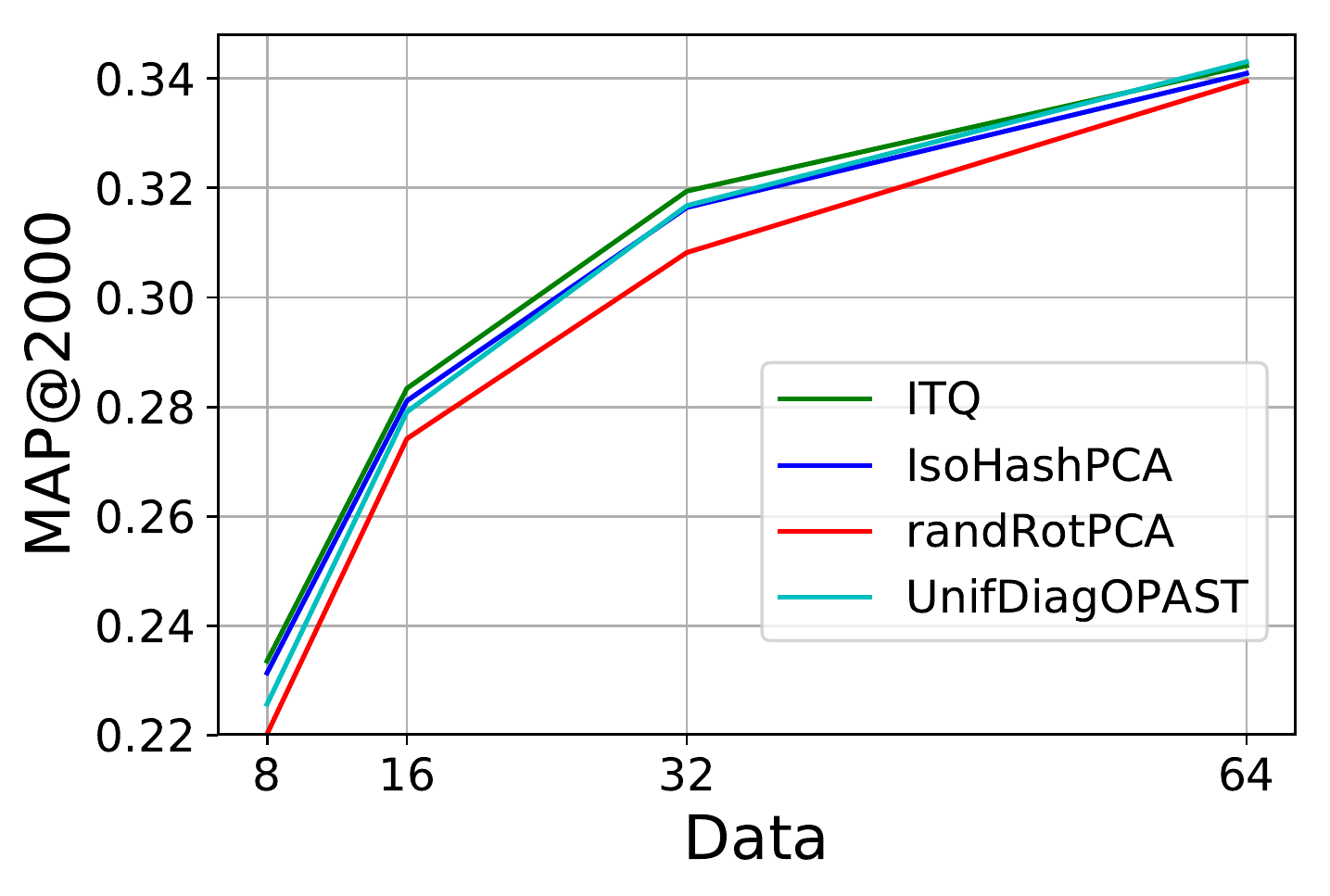}
  \caption{GIST}
  \label{subfig:batch_GIST}
\end{subfigure}
\caption{MAP@2000 in the batch setting for various code lengths.}
\label{fig:batch_comparison}
\end{figure}

\subsubsection{Comparison with online methods}
In the online setting, we print the MAP after every $5$ data points until the $3000$ because a plateau is then reached (but the NN are still computed over the whole training set).
We compared here four unsupervised online baseline methods that follow the basic hashing scheme $\Phi(x_t) = \operatorname{sgn}(\tilde{W}_t x_t)$, where the projection matrix $\tilde{W}_t \in \mathbb{R}^{c \times d}$ is determined according to the chosen method:
1) \textbf{OSH}~(\cite{LengWC0L15}): the number of chunks/rounds is set to $100$ and $l = 200$.
2) \textbf{RandRotOPAST}: $W_t$ is the PCA matrix obtained with OPAST and $R_t$ a constant random rotation.
3) \textbf{IsoHashOPAST}: $R_t$ is obtained with IsoHash. 
4) \textbf{UnifDiagOPAST}~(\cite{MorvanSGA18}).
Fig.~\ref{fig:online_comparison_CIFAR} and \ref{fig:online_comparison_GIST} (best viewed in color) shows the MAP for both datasets for different code length. Not surprisingly, UnifDiagOPAST and the online version of IsoHash with OPAST exhibit similar behavior for both datasets. Moreover, for small values of code length ($c < 64$), UnifDiagOPAST outperforms OSH and randRotOPAST while all have similar results for $c = 64$.

\begin{figure*}
\begin{subfigure}{0.25\textwidth}
  \centering
  \includegraphics[width=\linewidth]{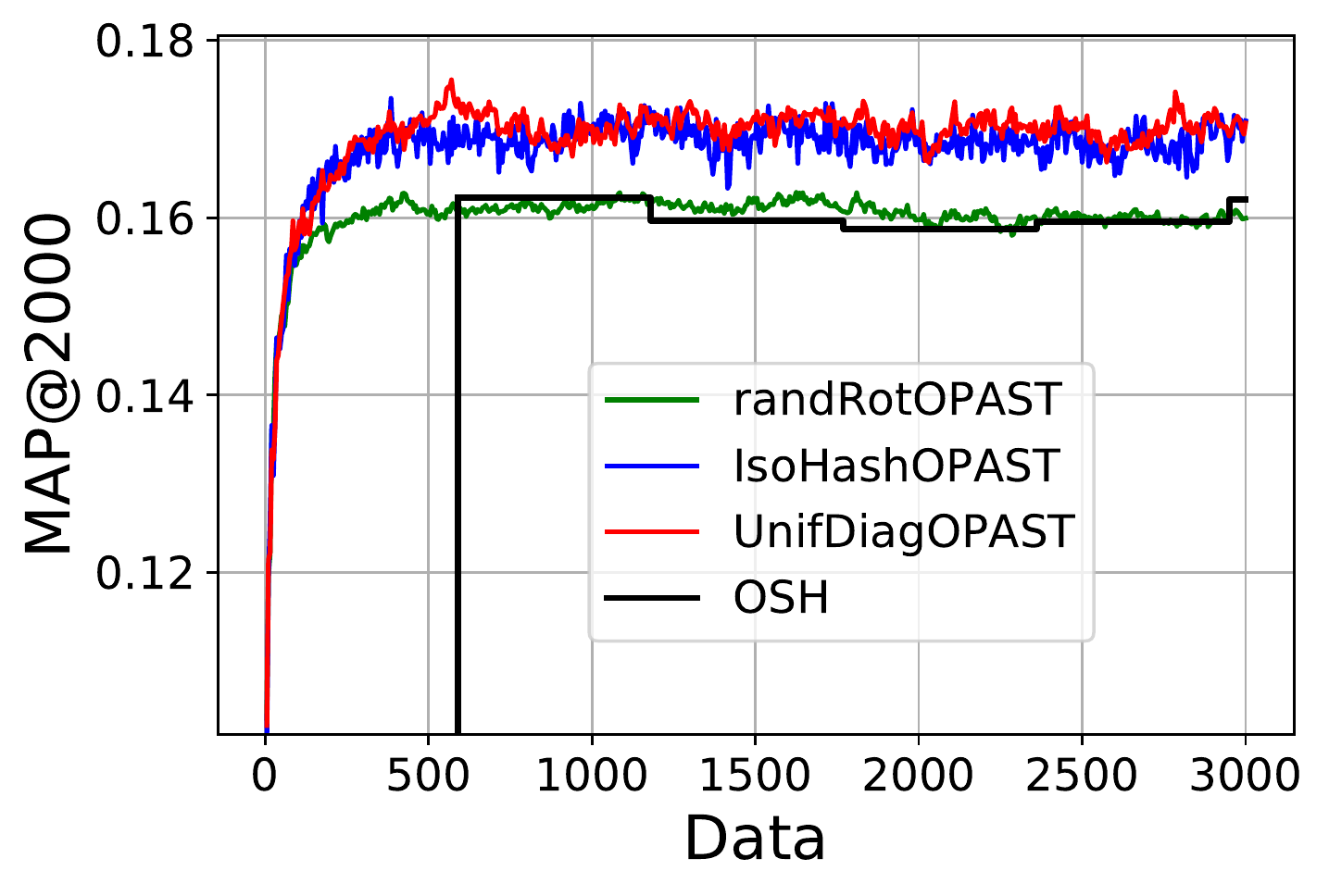}
  \caption{$c = 8$}
  \label{subfig:online_CIFAR_8}
\end{subfigure}%
\begin{subfigure}{0.25\textwidth}
  \centering
  \includegraphics[width=\linewidth]{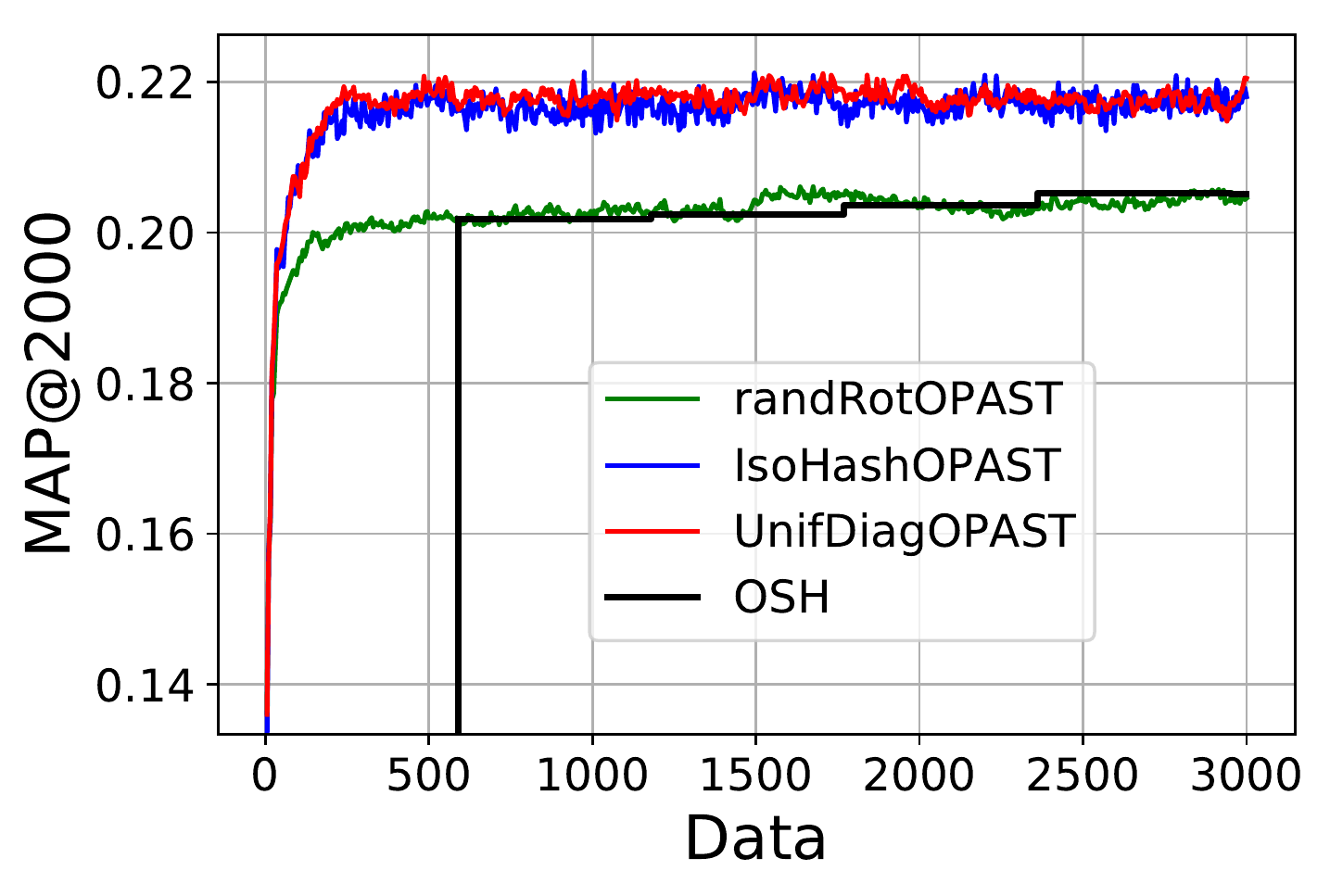}
  \caption{$c = 16$}
  \label{subfig:online_CIFAR_16}
\end{subfigure}%
\begin{subfigure}{0.25\textwidth}
  \centering
  \includegraphics[width=\linewidth]{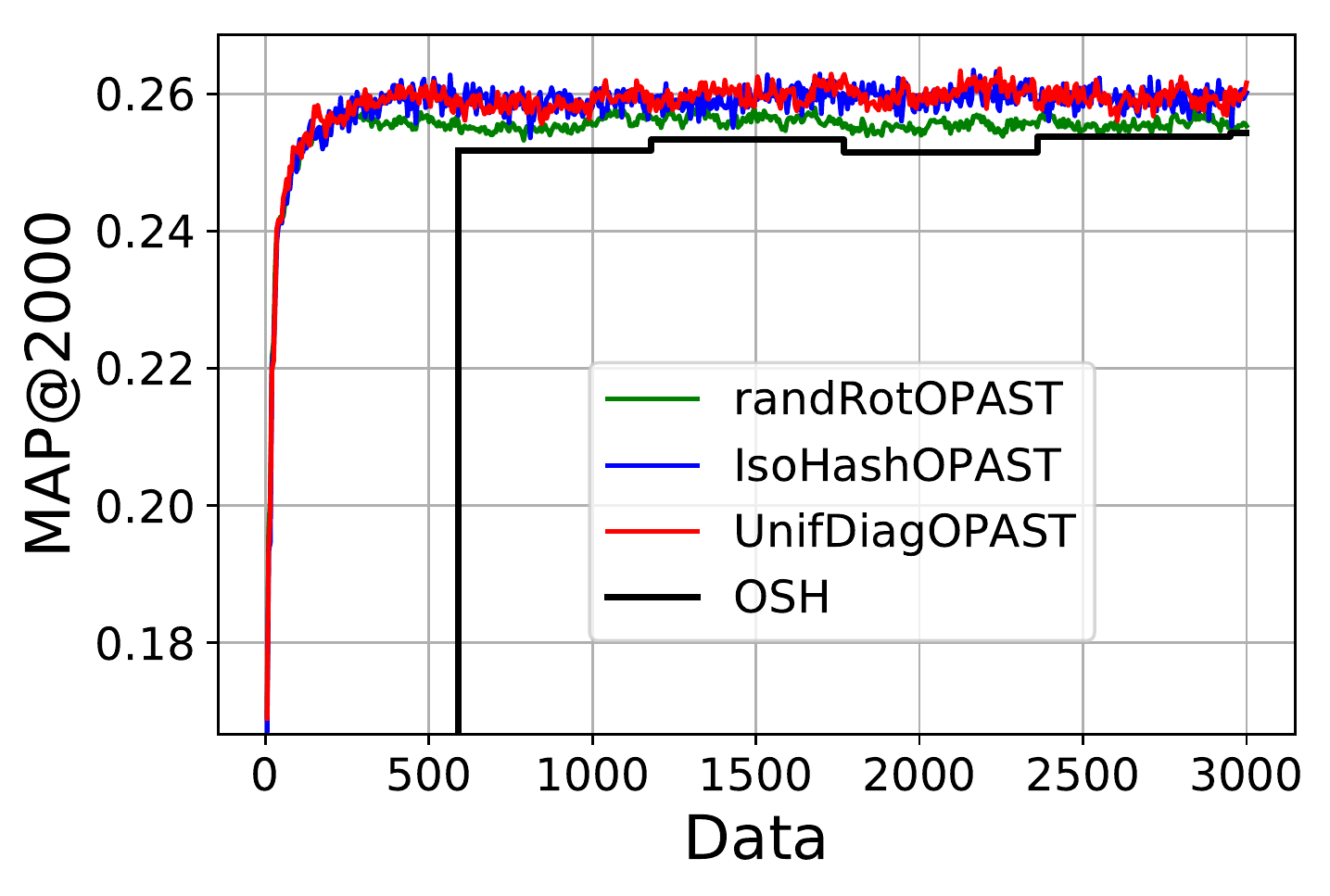}
  \caption{$c = 32$}
  \label{subfig:online_CIFAR_32}
\end{subfigure}%
\begin{subfigure}{0.25\textwidth}
  \centering
  \includegraphics[width=\linewidth]{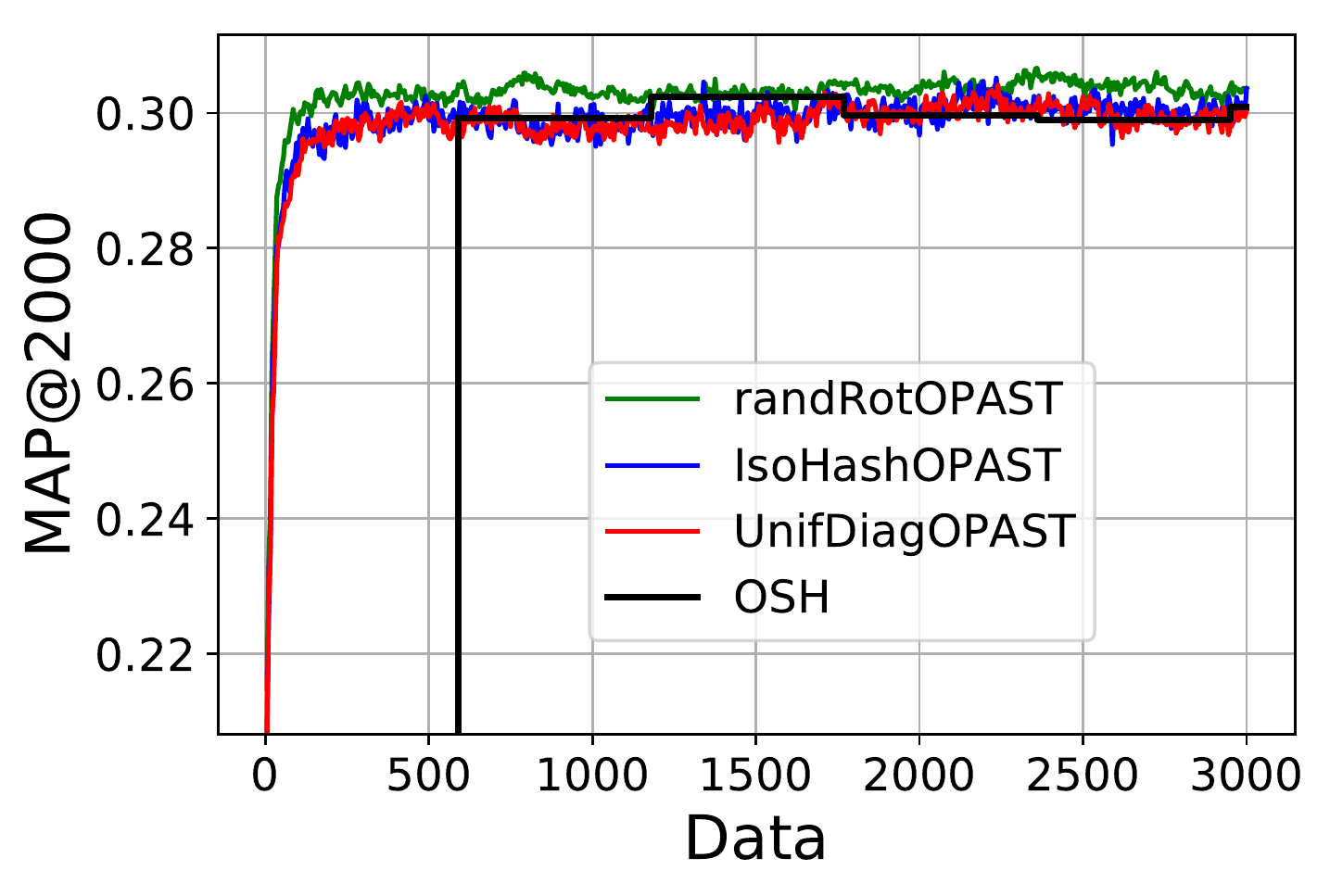}
  \caption{$c = 64$}
  \label{subfig:online_CIFAR_64}
\end{subfigure}
\caption{MAP@2000 in the online setting for different code length and CIFAR.}
\label{fig:online_comparison_CIFAR}
\end{figure*}

\begin{figure*}
\centering
\begin{subfigure}{0.25\textwidth}
  \centering
  \includegraphics[width=\linewidth]{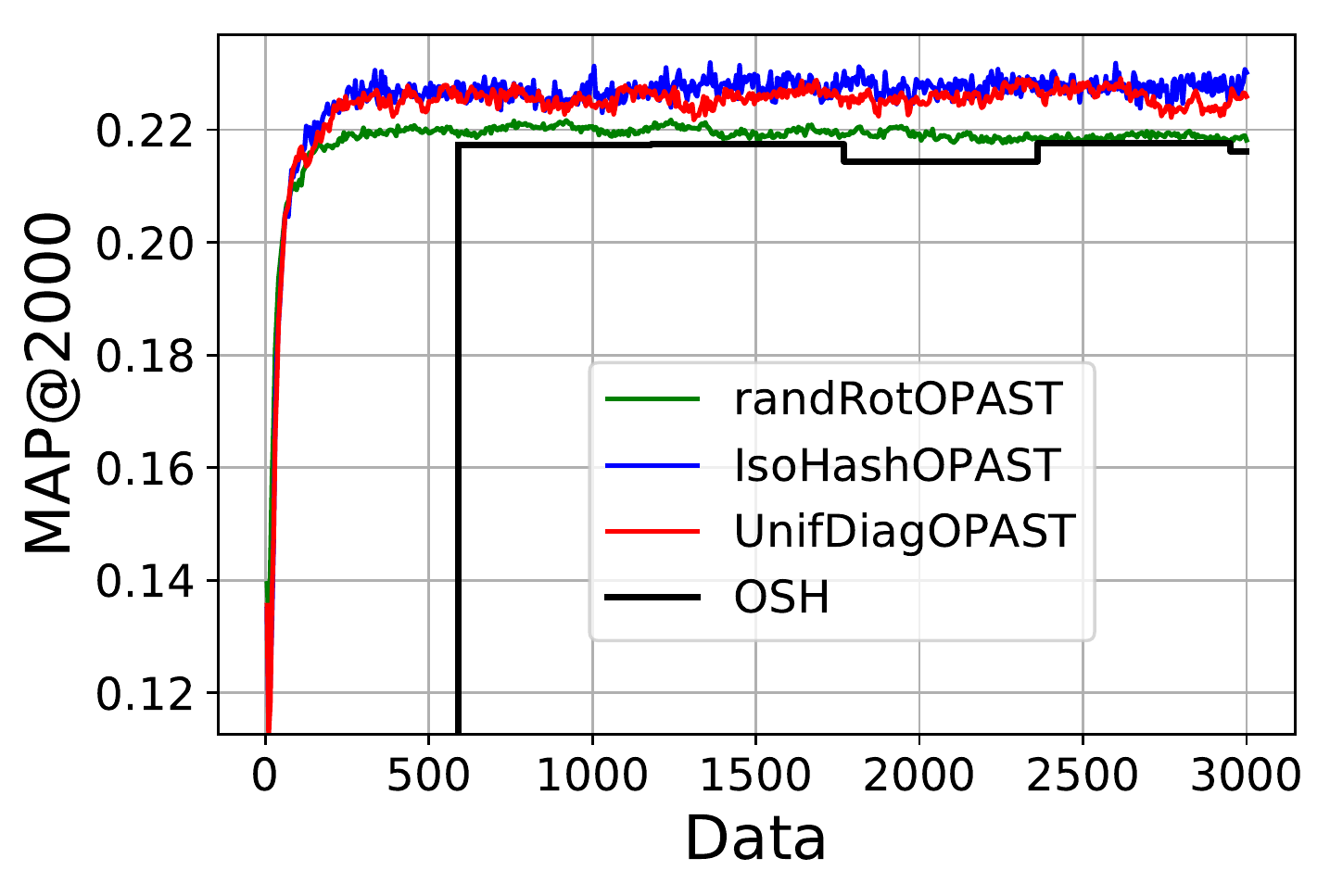}
  \caption{$c = 8$}
  \label{subfig:online_GIST_8}
\end{subfigure}%
\begin{subfigure}{0.25\textwidth}
  \centering
  \includegraphics[width=\linewidth]{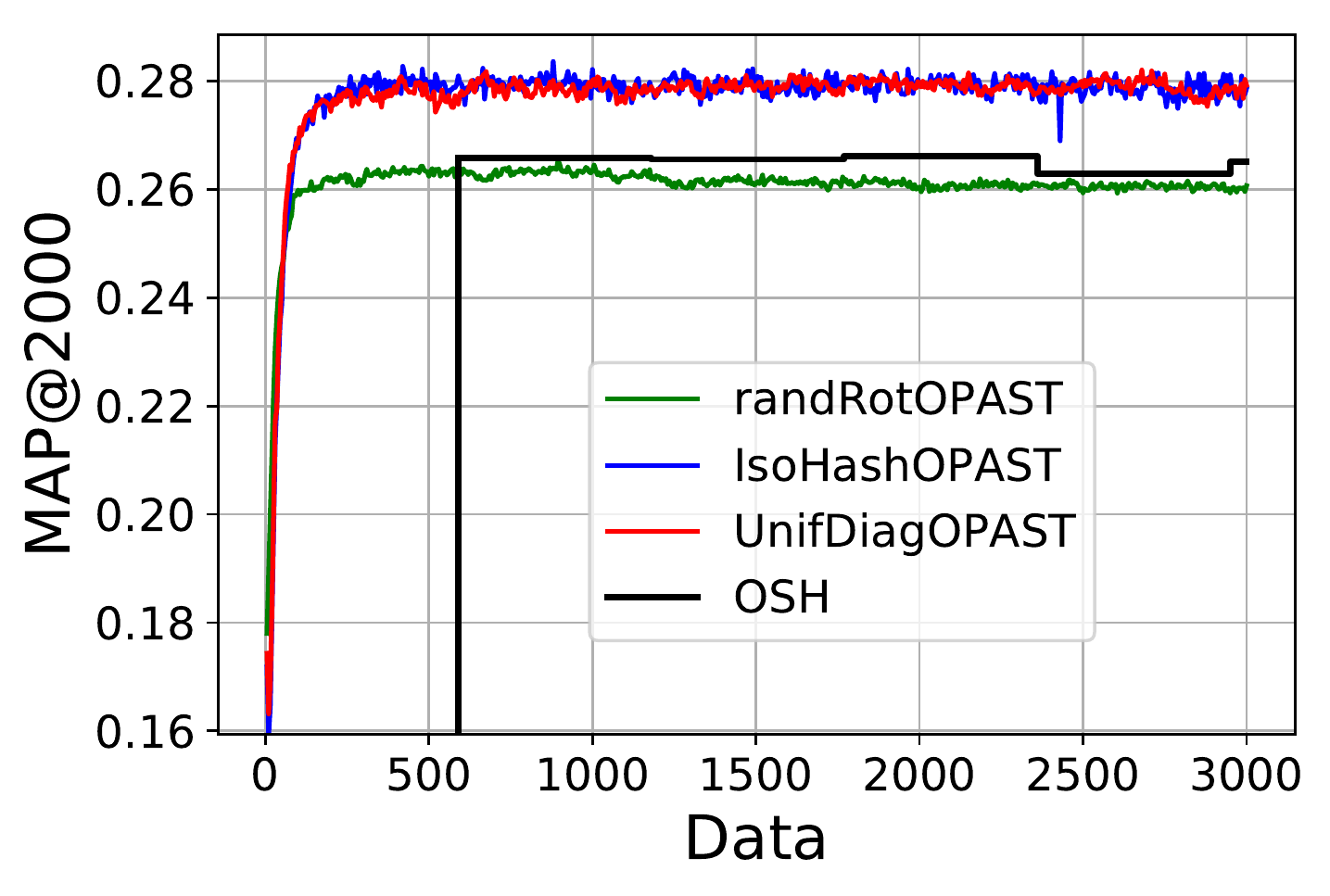}
  \caption{$c = 16$}
  \label{subfig:online_GIST_16}
\end{subfigure}%
\begin{subfigure}{0.25\textwidth}
  \centering
  \includegraphics[width=\linewidth]{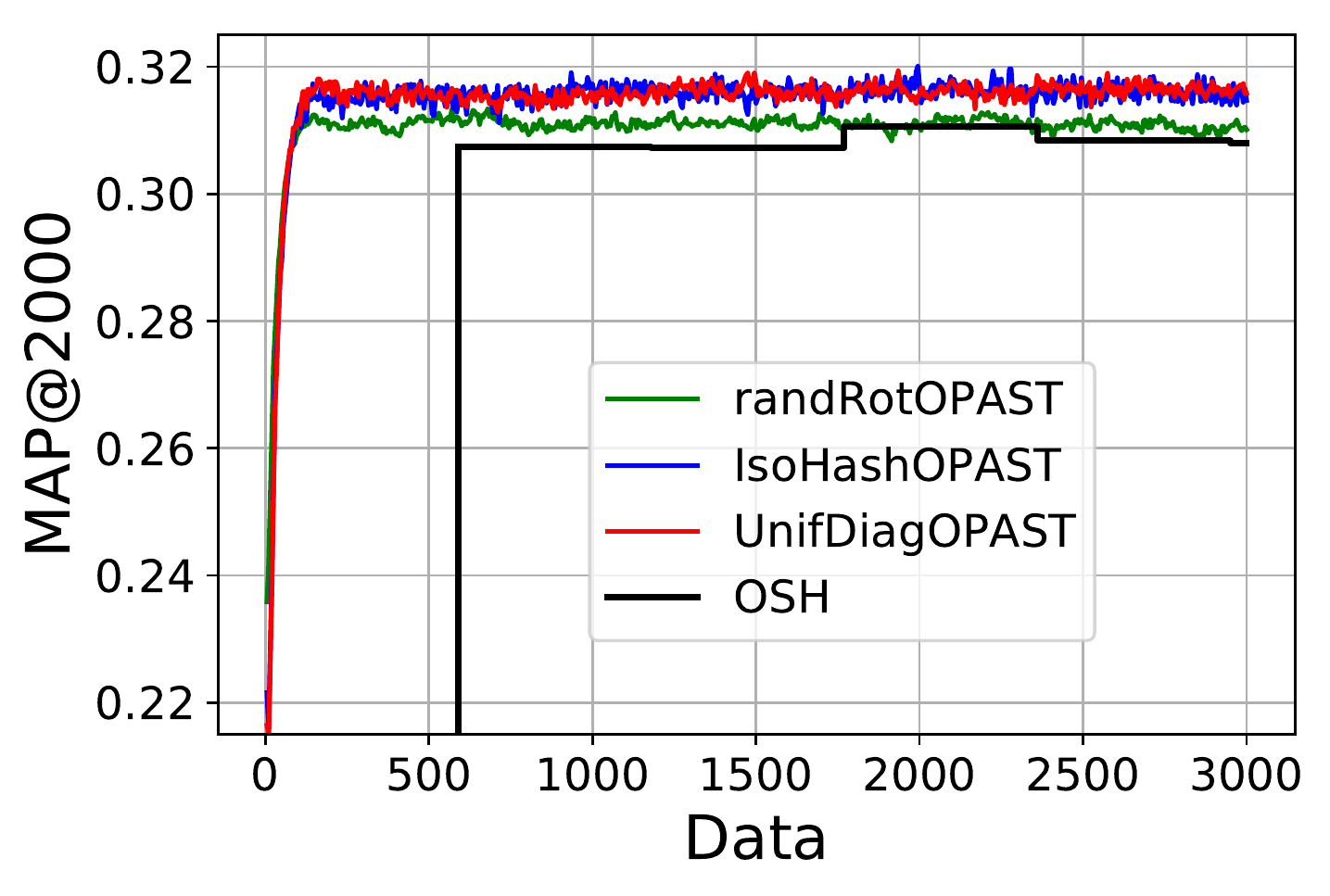}
  \caption{$c = 32$}
  \label{subfig:online_GIST_32}
\end{subfigure}%
\begin{subfigure}{0.25\textwidth}
  \centering
  \includegraphics[width=\linewidth]{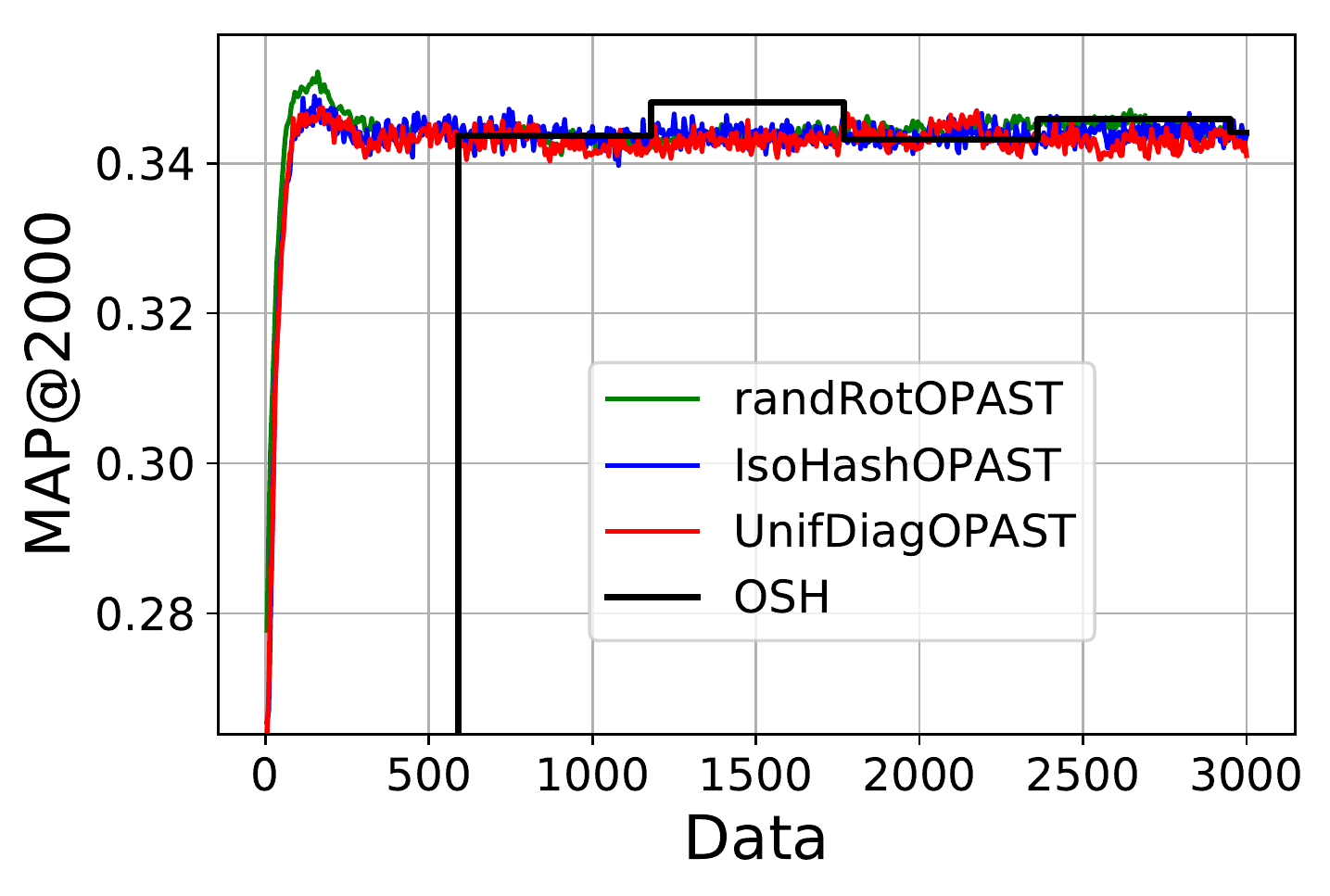}
  \caption{$c = 64$}
  \label{subfig:online_GIST_64}
\end{subfigure}
\caption{MAP@2000 in the online setting for different code length and GIST.}
\label{fig:online_comparison_GIST}
\end{figure*}

\subsection{Effect of rotation on binary sketches}

In this Section, in the offline context, the experiments shed light on the theory above by showing why a rotation gives better binary sketches than simply PCA projection. Rotations considered are random or learned from ITQ, IsoHash and UnifDiag. 
First, for CIFAR and GIST datasets, we compute the cumulative distribution function of $\mathbb{P}[ | y^{(i)}_t | < \epsilon]$, i.e. the probability for all $t \in [n]$ of $y_t \in \mathbb{R}^{c}$ to have entries near zero before and after the rotation application. Fig.~\ref{fig:proba_near_zero} plots $\mathbb{P}[ | y^{(i)}_t | < \epsilon]$ for $c = 32$ (averaged on $5$ runs). Similar results are obtained for other code lengths. For all rotation-based methods, this probability is always lower similarly than the one associated to only PCA projection.

\begin{figure*}
\centering
\begin{subfigure}{0.45\textwidth}
  \centering
  \includegraphics[width = 0.9876\linewidth]{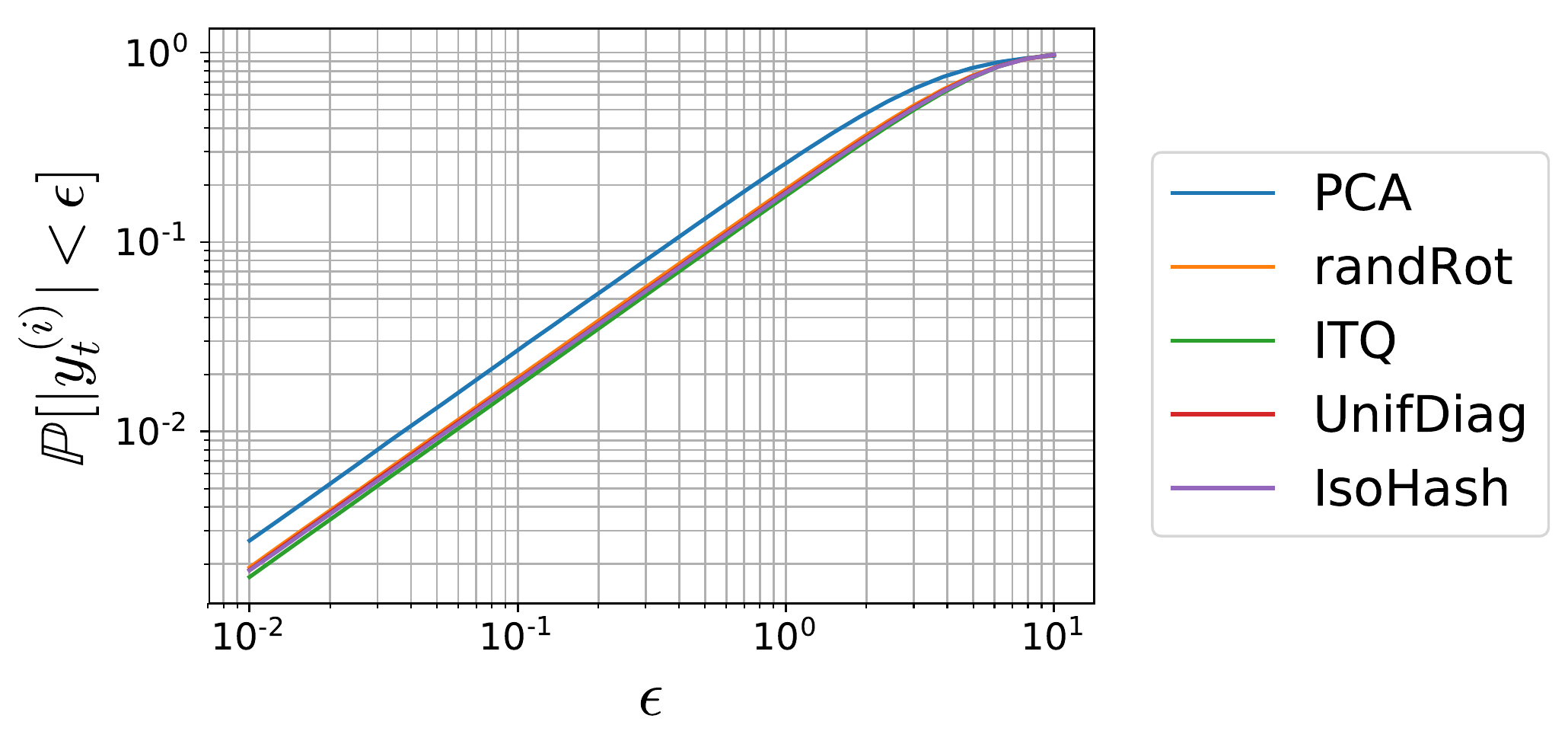}
  \caption{CIFAR}
   \label{subfig:proba_near_zero_CIFAR}
\end{subfigure}
\begin{subfigure}{0.45\textwidth}
  \centering
  \includegraphics[width = 0.9876\linewidth]{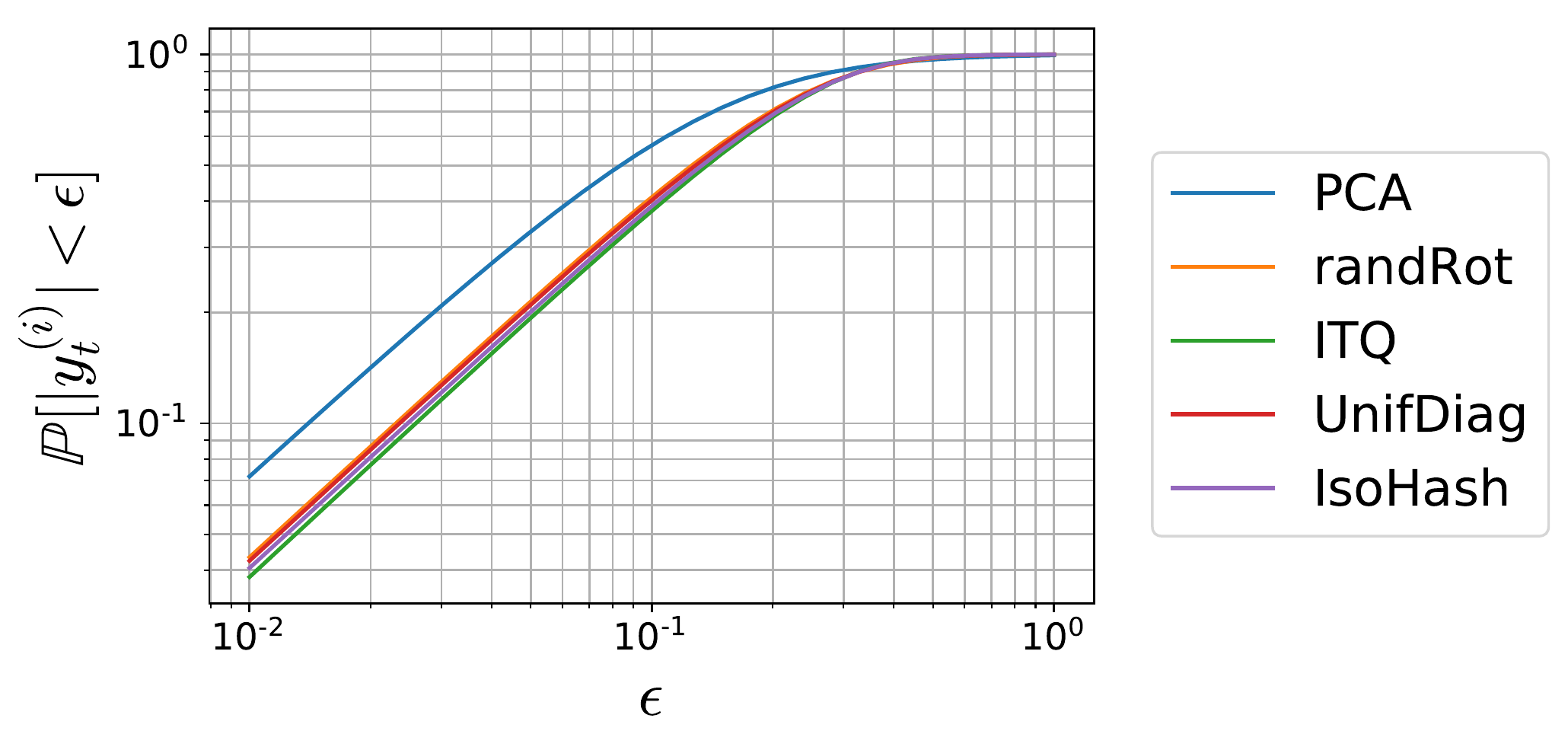}
  \caption{GIST}
   \label{subfig:proba_near_zero_GIST}
\end{subfigure}%
\caption{Cumulative distribution function for CIFAR and GIST and different hashing methods: $\forall i \in [c], \mathbb{P}[ |y^{(i)}_t| < \epsilon]$ for $c = 32$.}
  \label{fig:proba_near_zero} 
\end{figure*}

Secondly, we provide a visualization of the rotation efficiency in the clustering task. This experiment is made on simulated data since a ground truth partition is required.
We consider $C$ equally distributed clusters of $n$ data points such that nearest neighbors of a data point are points from the same cluster. We choose the centroids from these clusters randomly. The expected result is a small variance of the binary codes within the same cluster. Fig.~\ref{fig:bar_code_gaussian_blobs} displays the binary sketches obtained with and without rotation for $C = 6$, $n = 6000$ points with $d = 960$ and $c = 32$. Each column is a binary sketch: a yellow case represents a bit equal to $1$ and a red pixel stands for $-1$. Each cluster, delimited by a blue vertical line, contains $1000$ points plotted in order.
An interpretation of the results is that the rotation tends to move data away from the orthant since more binary sketches after application of the rotation have bits in common. This is illustrated by the obtained blocks of the same color, as opposed to the ``blurry" visualization implied by the PCA alone.
\begin{figure*}[p]
\centering
\begin{subfigure}{0.3\textwidth}
  \centering
  \includegraphics[width=0.9875\linewidth]{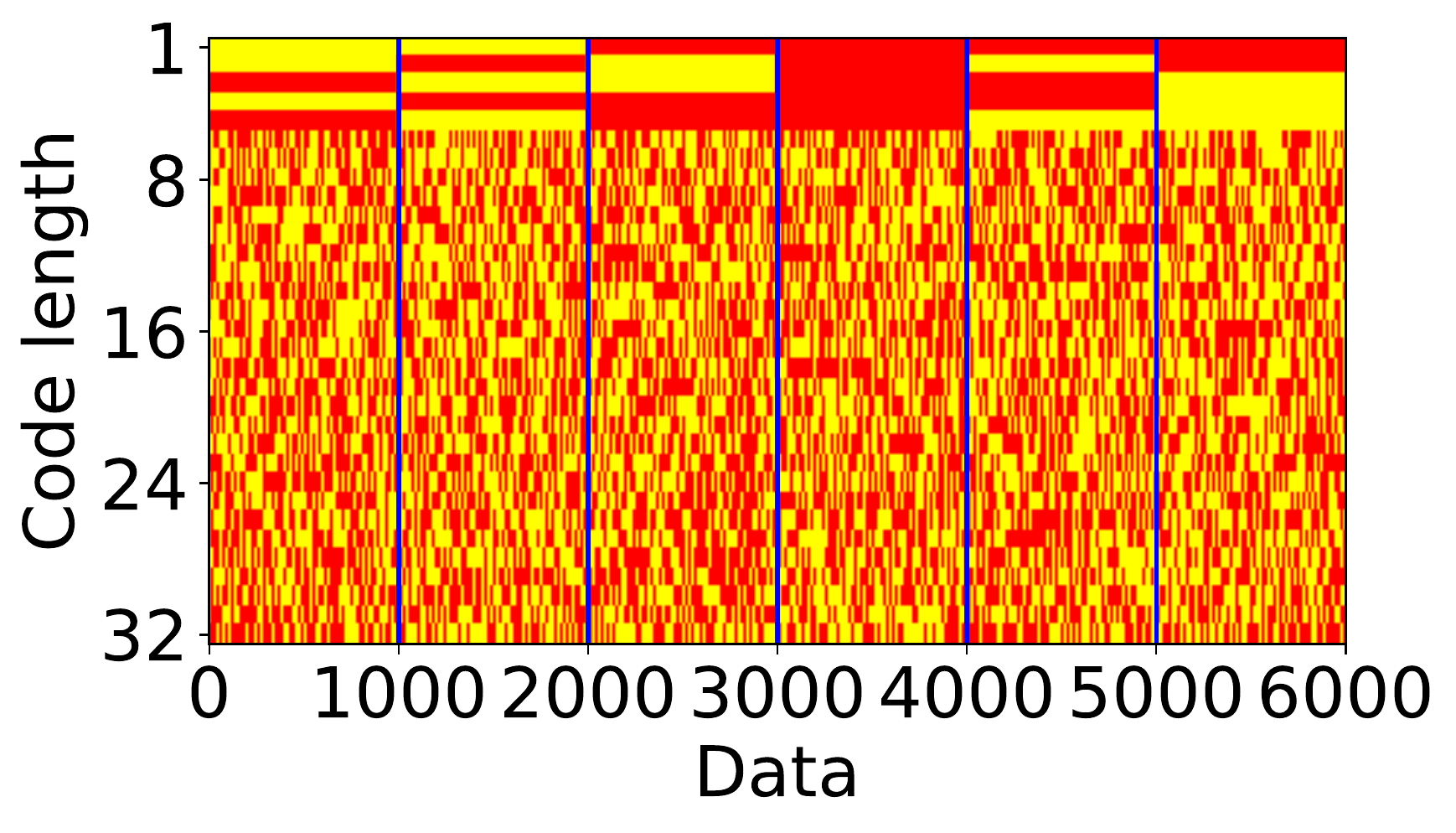}
  \caption{PCA}
  \label{subfig:gaussian_blobs_bar_code_init_c32_seed1}
\end{subfigure}
\begin{subfigure}{0.3\textwidth}
  \centering
  \includegraphics[width=0.9875\linewidth]{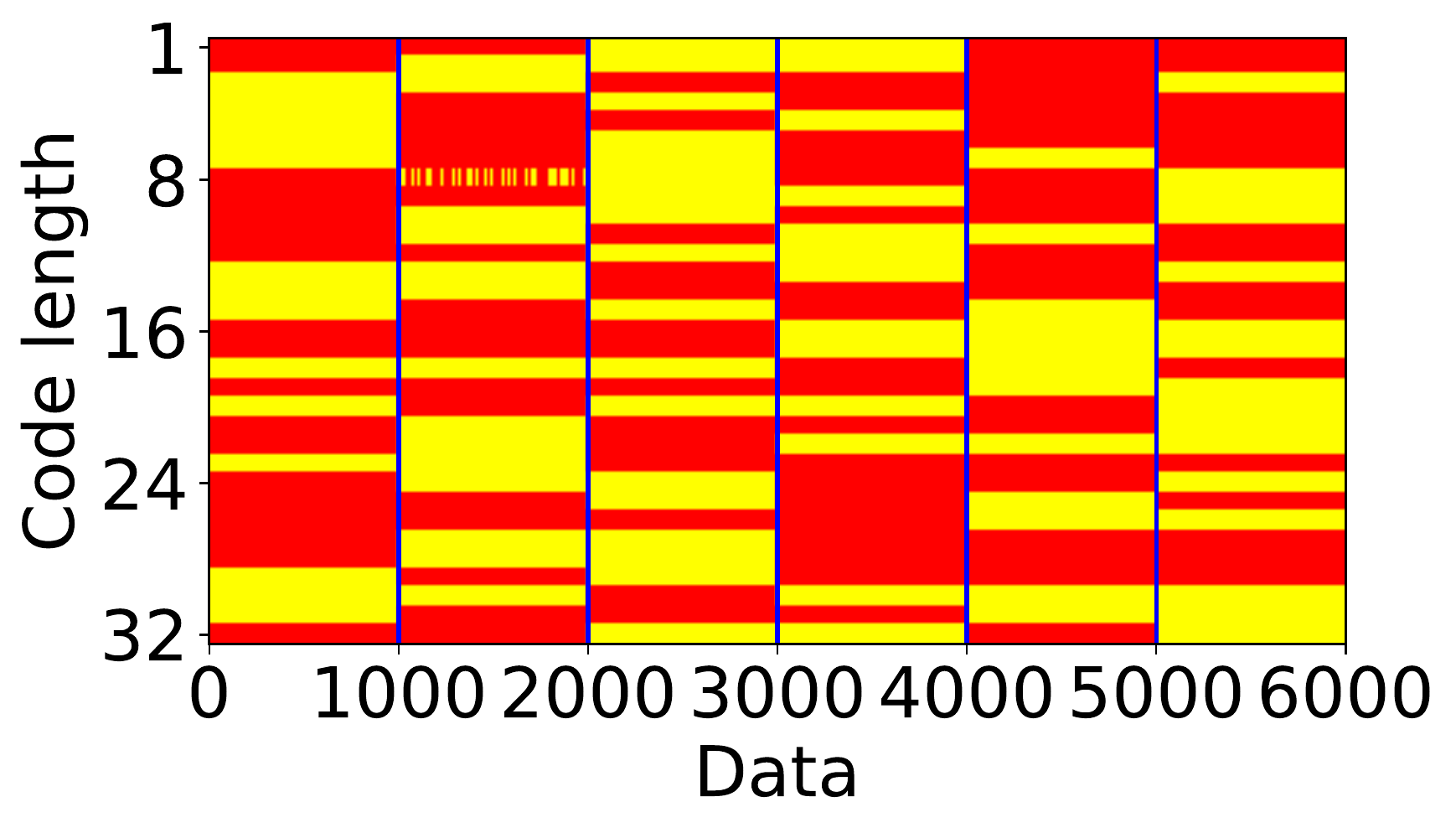}
  \caption{RandRot}
\label{subfig:gaussian_blobs_bar_code_rot_randRot_c32_seed1}
\end{subfigure}
\begin{subfigure}{0.3\textwidth}
  \centering
  \includegraphics[width=0.9875\linewidth]{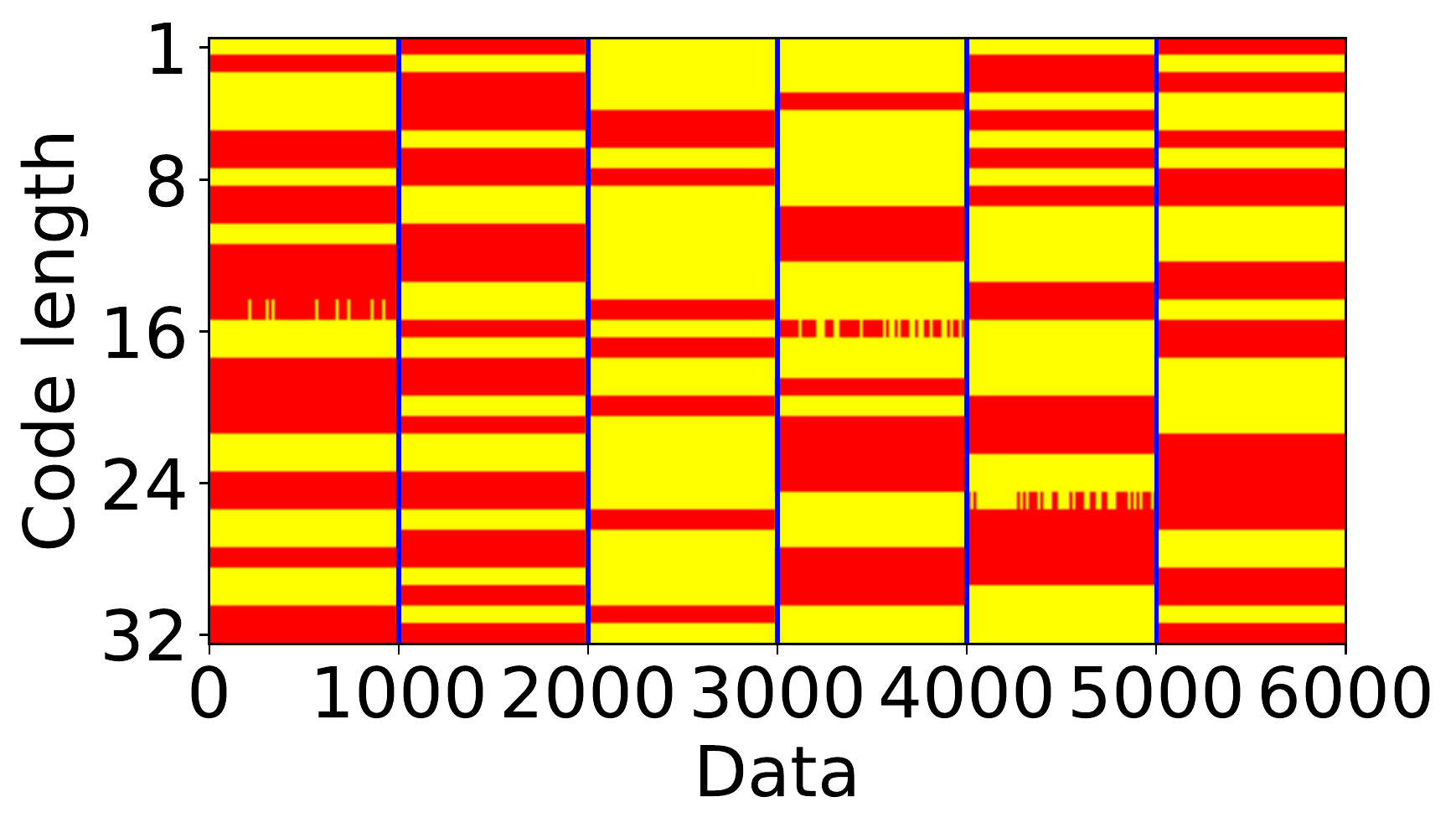}
  \caption{ITQ}
  \label{subfig:gaussian_blobs_bar_code_rot_ITQ_c32_seed1}
\end{subfigure}
\begin{subfigure}{0.3\textwidth}
  \centering
  \includegraphics[width=0.9875\linewidth]{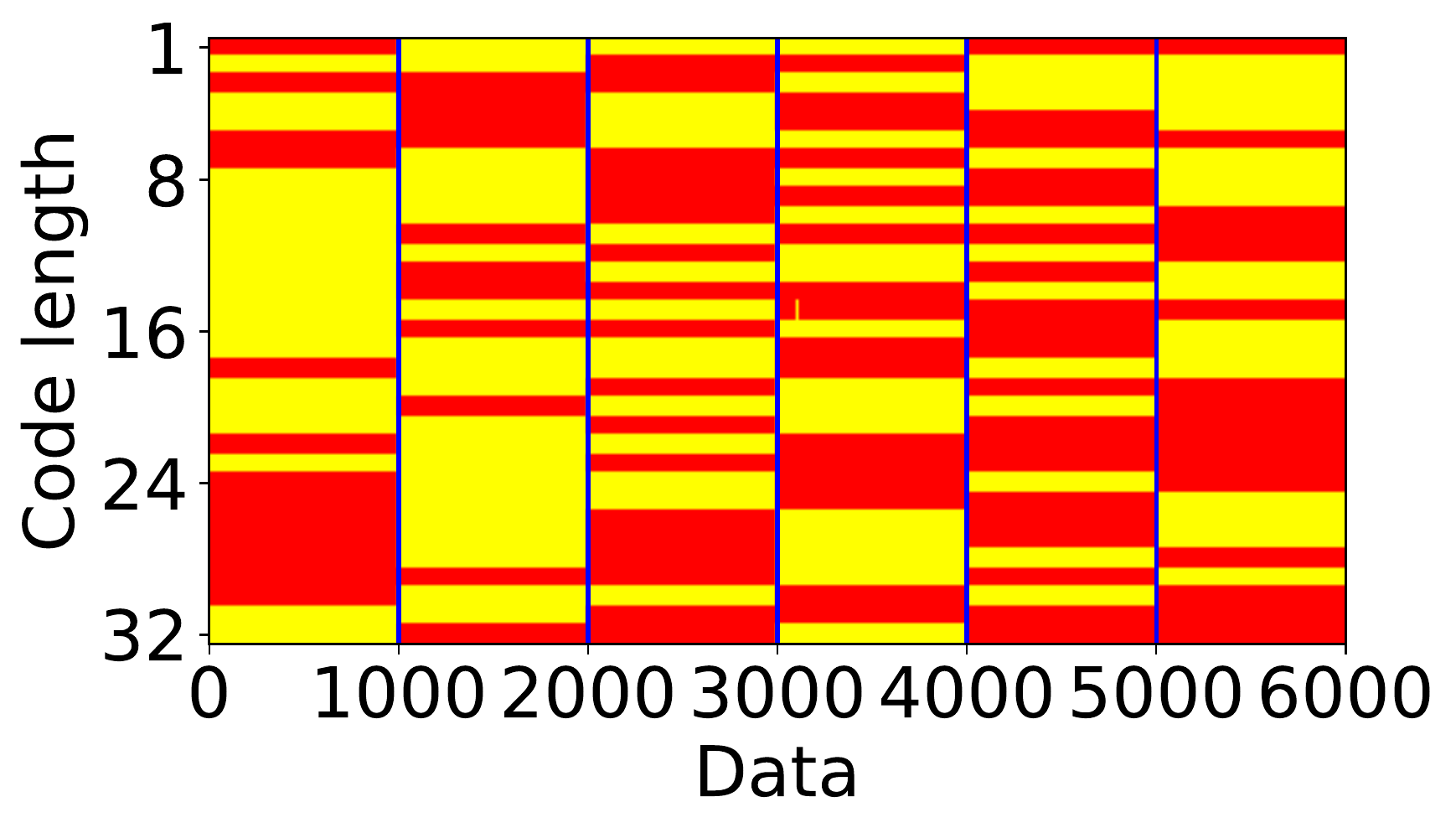}
  \caption{IsoHash} \label{subfig:gaussian_blobs_bar_code_rot_IsoHash_c32_seed1}
\end{subfigure}
\begin{subfigure}{0.3\textwidth}
  \centering
  \includegraphics[width=0.9875\linewidth]{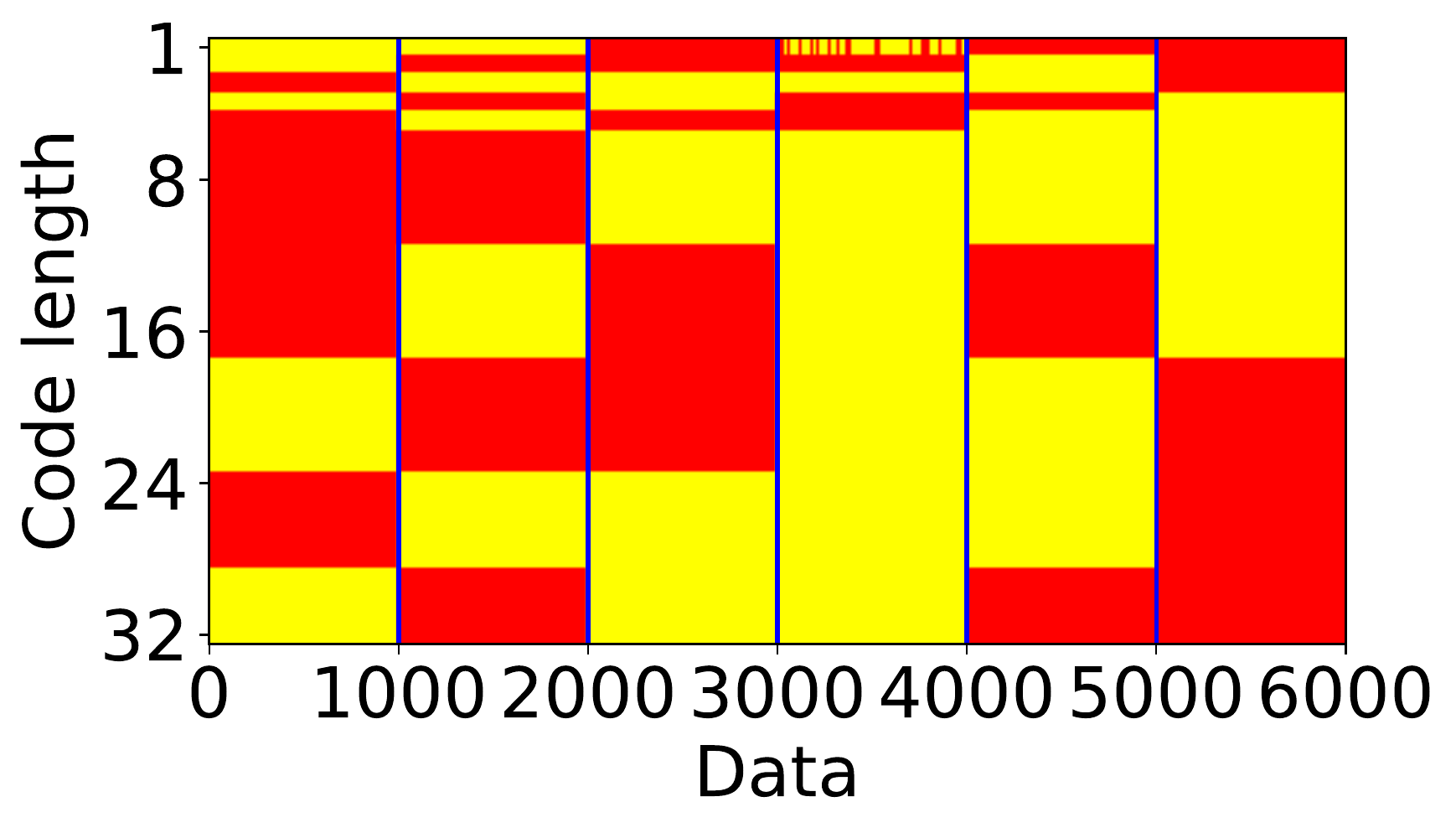}
  \caption{UnifDiag}
\label{subfig:gaussian_blobs_bar_code_rot_UnifDiag_c32_seed1}
\end{subfigure}  
\caption{Effect of rotation on binary sketches on simulated data: $6$ clusters with  $d = 960$ and $c= 32$.}
\label{fig:bar_code_gaussian_blobs}
\end{figure*}
More quantitatively, Table~\ref{table:mean_var_all} compiles the variance of the binary sketches averaged on the $6$ clusters for $10$ partitions. Not surprisingly, PCA gives the worst results: the high variance in the binary sketches explains the ``blurry" previous visualization. Conversely, all rotation-based methods tend to reduce the variance in the binary sketches. 

\begin{table}[t]
\caption{Mean variance for the binary sketches (avg. on  $10$ runs) obtained for $6$ convex clusters with random centroids in $d = 960$.}
\label{table:mean_var_all}
\begin{center}
\begin{small}
\begin{sc}
\begin{tabular}{lccccr}
\toprule
 & $8$ & $16$ & $32$ & $64$ \\
\midrule                 
PCA      & $7.1 \times 10^{-2}$ & $6.5 \times 10^{-2}$ & $4.0 \times 10^{-2}$ & $2.2 \times 10^{-2}$ \\
randRot  & $3.9 \times 10^{-4}$ & $2.6 \times 10^{-4}$ & $1.9 \times 10^{-4}$ & $7.4 \times 10^{-5}$ \\
ITQ      & $0.0$ & $2.0 \times 10^{-4}$ & $1.1 \times 10^{-4}$ & $1.3 \times 10^{-4}$ \\
UnifDiag & $4.1 \times 10^{-4}$ & $2.2 \times 10^{-4}$ & $2.5 \times 10^{-4}$ & $1.1 \times 10^{-4}$ \\
IsoHash  & $1.4 \times 10^{-4}$ & $3.3 \times 10^{-4}$ & $2.1 \times 10^{-4}$ & $1.3 \times 10^{-4}$ \\
\bottomrule
\end{tabular}%
\end{sc}
\end{small}
\end{center}
\vskip -0.1in
\end{table}

\section{Conclusion} 
State-of-the-art unsupervised hypercubic quantization hashing methods preprocess data with Principal Component Analysis and then rotate them data to balance variance over the different directions. It has been shown to give better sketches accuracy in the nearest neighbors search task but to the best of our knowledge, this is the first time that theoretical guarantees are provided. In particular, rotations uniformizing the diagonal of the data covariance matrix are interesting since they enable to deploy the sketching algorithm in the streaming setting. As there exists possibly an infinity of rotations uniformizing the covariance matrix diagonal, further investigation would be to evaluate among them which ones perform better.  

\bibliography{on_the_needs_biblio}
\bibliographystyle{abbrvnat}

\end{document}